%% file: main.tex
\documentclass[conference]{IEEEtran}
\IEEEoverridecommandlockouts
\usepackage{cite}
\usepackage{amsmath,amssymb,amsfonts}
\usepackage{algorithm}
\usepackage{graphicx}
\usepackage{xcolor}
\usepackage{multirow}
\usepackage{float}
\usepackage{subfigure}
\usepackage{xspace}
\usepackage{ntheorem}
\usepackage{threeparttable}

\usepackage{algpseudocode}
\usepackage{caption} 
\usepackage{color, soul}
\usepackage{arydshln}

\newcommand{\tabincell}[2]{\begin{tabular}{@{}#1@{}}#2\end{tabular}}

\newcommand{\sys}{{UniNet}\xspace}

\newtheorem{theorem}{Theorem}
\newtheorem*{proof}{Proof}
\newtheorem{lemma}{Lemma}

\def\BibTeX{{\rm B\kern-.05em{\sc i\kern-.025em b}\kern-.08em
    T\kern-.1667em\lower.7ex\hbox{E}\kern-.125emX}}

\begin{document}
\title{\sys: Scalable Network Representation Learning with Metropolis-Hastings Sampling}

\author{\IEEEauthorblockN{Xingyu Yao\textsuperscript{\#}\thanks{Xingyu Yao and Yingxia Shao contributed equally to this work.} ~~~~Yingxia Shao\textsuperscript{\#}\textsuperscript{*}\thanks{*Corresponding author.} ~~~~Bin Cui\textsuperscript{$\natural$} ~~~~Lei Chen\textsuperscript{$\flat$}}
\IEEEauthorblockA{\textsuperscript{\#}Beijing University of Posts and Telecommunications}
\IEEEauthorblockA{\textsuperscript{$\natural$}School of EECS, Peking University}
\IEEEauthorblockA{\textsuperscript{$\flat$}Hong Kong University of Science and Technology} 
\{georgia\_str, shaoyx\}@bupt.edu.cn, bin.cui@pku.edu.cn, leichen@cs.ust.hk\\
}

\maketitle              
\begin{abstract}
Network representation learning (NRL) technique has been successfully adopted in various data mining and machine learning applications.
Random walk based NRL is one popular paradigm, which uses a set of random walks to capture the network structural information, 
and then employs word2vec models to learn the low-dimensional representations.
However, until now there is lack of a framework, which unifies existing random walk based NRL models 
and supports to efficiently learn from large networks.
The main obstacle comes from the diverse random walk models and the inefficient sampling method for the random walk generation.
In this paper, we first introduce a new and efficient edge sampler
based on Metropolis-Hastings sampling technique, 
and theoretically show the convergence property of the edge sampler to arbitrary discrete probability distributions.
Then we propose a random walk model abstraction,
in which users can easily define different transition probability by specifying dynamic edge weights and random walk states.
The abstraction is efficiently supported by our edge sampler, since our sampler can draw samples from unnormalized probability distribution in constant time complexity.
Finally, with the new edge sampler and random walk model abstraction, 
we carefully implement a scalable NRL framework called \sys.
We conduct comprehensive experiments with five random walk based NRL models over eleven real-world datasets,
and the results clearly demonstrate the efficiency of \sys over billion-edge networks. 
The code of \sys is released at: https://github.com/shaoyx/UniNet.
\end{abstract}

\input{1.intro_new.tex}
\input{2.prel.tex}
\input{3.tech_new.tex}
\input{4.sys.tex}

\input{expr.tex}
\input{relate.tex}

\input{conc.tex}
\input{appendix.tex}
%
%
\bibliographystyle{plain}
\bibliography{main.bib}

\end{document}

%% file: 1.intro_new.tex
\section{Introduction}
\label{sec:intro}
{Over the last few years, network representation learning (NRL)~\cite{bengio2013representation} has become a de-facto tool for modeling the latent structural information of the network. It has been used for an array of data mining and machine learning tasks, including network clustering~\cite{OPSAHL2009155}, community detection~\cite{FORTUNATO201075}, node classification~\cite{kipf2016semi}, link prediction~\cite{zhang2018link}, and so on. Random walk based NRL is one of the notable NRL models.
The general procedure of random walk based NRL can be written as the following two steps:
\begin{enumerate}
    \item {\em Walks = RandomWalkGeneration($\mathcal{G}$, N, L)};
    \item {\em Embeddings = Word2Vec(Walks)}.
\end{enumerate} 
The function \textit{RandomWalkGeneration} is responsible for creating $N$ random walks of length $L$ for each node in the network $\mathcal{G}$, while the function \textit{Word2Vec} learns the node embedding from the random walk sequences with word2vec models (e.g., skipgram~\cite{mikolov2013distributed}, cbow~\cite{mikolov2013efficient}). 
Table~\ref{tab:existing_models} lists five prominent random walk based NRL models.
Though they follow the same two steps, none of them are developed in a common framework,
where they share an optimized engine under the hood. 
What's worse, existing realizations, like the open-sourced implementations in Table~\ref{tab:existing_models}, 
cannot handle large-scale networks due to the memory explosion problem or the expensive training cost.}

\begin{table*}[t]
    \centering
    \caption{Characteristics of five popular random walk based NRL models.\label{tab:existing_models}}
  \scalebox{0.9}{
      \begin{tabular}{|l|l|l|l|l|}
      \hline
      Model &  Transition probability distribution of a node $v$ & \#states & Network type & links of the open-sourced version \\
        \hline
        \hline
        deepwalk~\cite{perozzi2014deepwalk} &
        \(
            \mathcal{G}_v(u) = \frac{w_{vu}}{\sum_{k\in \mathcal{N}(v)} w_{vk}} 
        \)
        &
        $|\mathcal{V}|$
        & Homogeneous
        & https://github.com/phanein/deepwalk \\
        \hline
        node2vec~\cite{grover2016node2vec}&
        $
            \mathcal{G}_{s,v}(u) = \frac{\alpha_u w_{vu}}{\sum_{k\in \mathcal{N}(v)}\alpha_k w_{vk}} 
        $
        &
        $|\mathcal{E}|$
        & Homogeneous
        & \tabincell{l}{https://github.com/aditya-grover/node2vec [orignal] \\ https://github.com/eliorc/node2vec [parallel]} \\
        \hline
        metapath2vec~\cite{metapath2vec} & 
        \(
            \mathcal{G}_{T,v}(u) = \left\{ \begin{array}{lcl}{\frac{w_{vu}}{\sum_{k\in \mathcal{N}(v) \land \Phi(k)=T } w_{vk}}} & {\Phi(u)=T}, \\
            0 & {\Phi(u) \not= T} \end{array}\right.
        \)
        &
        $|\mathcal{V}||\Phi|$
        & Heterogeneous
        & https://ericdongyx.github.io/metapath2vec/m2v.html \\
        \hline
        edge2vec~\cite{edge2vec} & 
        \(
            \mathcal{G}_{s,v}(u) = \frac{\alpha_u\mathbf{M}_{\Phi(s,v),\Phi(v,u)} w_{vu}}{\sum_{k\in \mathcal{N}(v)}\alpha_k\mathbf{M}_{\Phi(s,v),\Phi(v,k)} w_{vk}} 
        \)
        &
        $|\mathcal{E}|$
        & Heterogeneous
        & https://github.com/RoyZhengGao/edge2vec \\
        \hline
        fairwalk~\cite{fairwalk} & 
        \(
            \mathcal{G}_{s,v}(u) = \frac{\alpha_u w_{vu}}{|\Phi|\sum_{k\in \mathcal{N}(v) \land \Phi(k)=\Phi(u)} \alpha_k w_{vk}} 
        \)
        &
        $|\mathcal{E}|$
        & Attributed Network
        &\tabincell{l}{https://github.com/EnderGed/Fairwalk [orignal] \\ https://github.com/urielsinger/fairwalk [parallel]}  \\
      \hline
      \end{tabular}%
      }
\end{table*}

Before presenting the challenges of designing a unified and scalable framework for random walk based NRL models, 
we introduce some backgrounds about the random walk over networks. 
{Random walk can be treated as a sequence of steps of decision making process conducted by a walker. At each step, the walker uses an \textit{edge sampler} to obtain an edge as the next direction to go. The edge sampler samples an edge from a transition probability distribution $\mathcal{G}_{\mathbf{x}}$ which defines the probability of selecting each edge under the consideration of the state $\mathbf{x}$ of the walker. 
The state $\mathbf{x}$ stores the data (e.g., the previous steps of the walker, current node, node type, etc.) that helps the walker identify the transition probability distribution and varies across different random walk models.
Taking the first-order random walk model in Deepwalk~\cite{perozzi2014deepwalk} as an example, since a  transition probability distribution is only related to the current node of a random walk, the state $\mathbf{x}$ is the current node, and the probability of selecting an edge under the consideration of the state $\mathbf{x}$ is defined by Eq.~\ref{eq:deepwalk_prob}.
Another important thing is that different random walk models have different number of states ($\#state$) over a network.
Since each state corresponds to a transition probability distribution, we need to handle $\#state$ different transition probability distributions during the random walk generation.
Table~\ref{tab:existing_models} also lists the number of states for the five NRL models.
}

{\em \textbf{Challenge 1}: the scalability of random walk based NRL models.}
Previous empirical studies~\cite{zhou2018efficient,tsitsulin2018verse} reveal that the random walk generation step is the bottleneck of the scalability problem\footnote{For the embedding learning step, in NLP or machine learning community, researchers have proposed optimization techniques to improve the efficiency of training word embedding models~\cite{8663393,10.1145/2983323.2983361}}. The efficiency of random walk generation is highly related to the performance of edge samplers. 
Existing edge samplers entail either expensive memory or expensive sampling cost over large networks, resulting in the inefficiency of random walk generation.
Assume the average node degree of the network is $d$. Alias edge sampler~\cite{walker1977efficient} for each transition probability distribution requires $\mathcal{O}(d)$ memory complexity and $\mathcal{O}(1)$ time complexity. {To manage all $\#state$ transition probability distributions, alias edge sampler entails $\mathcal{O}(d\#state)$ memory complexity if we materialize all the distributions. Direct edge sampler~\cite{direct_sampling} entails $\mathcal{O}(1)$ memory complexity but $\mathcal{O}(d)$ time complexity for each edge sampling.}

Recently, Yang et al.~\cite{yang2019knightking} introduced a rejection edge sampler~\cite{walker1977efficient}, which conducts sampling over a simple proposal distribution instead of the target distribution, and accepts the sampling result with an acceptance ratio $\theta$, which depends on the proximity between the proposal distribution and the target distribution. The time complexity of the sampler is $\mathcal{O}(C)$, where $C$ is a constant smaller than $d$ and is related to the acceptance ratio $\theta$; the memory cost is $\mathcal{O}(1)$. 
Since the acceptance ratio $\theta$ is sensitive to the NRL model parameters, we have a high variance in the efficiency of the model. 
Table~\ref{tab:const} shows the average acceptance ratio and time cost when running node2vec on Flickr with different model parameters. 
If the model parameters change from $(1,1)$ to $(0.25,1)$, the acceptance ratio drops from 1 to 0.25, incurring the sampling efficiency to be 2.60X slower.  

Therefore, to solve the scalability challenge, we need an edge sampler which is both efficient in time and memory aspects.

{\em \textbf{Challenge 2}: the diversity of random walk based NRL models.}
The diversity comes from the various definitions of transition probability distributions. 
From  Table~\ref{tab:existing_models}, we can see that deepwalk uses static edge weights ($w$) to define the probability, while node2vec uses the biased edge weights ($\alpha w$), where $\alpha$ is a factor determined by the hyper-parameters of node2vec. Metapath2vec further requires the node labels ($\Phi(u)$) for the computation of probability. 
{To sum up, different NRL models use distinct definitions of the \textit{dynamic edge weights} (e.g., $\alpha w$) to specify 
the \textbf{unnormalized} transition probability distributions in their random walk models.
When applying the aforementioned edge samplers to execute the sampling actions, they 
require the normalized factor of a distribution (e.g, $\sum_{k\in \mathcal{N}(v)}\alpha_k w_{vk}$ in node2vec).
However, it is expensive to obtain such normalized factors, especially on large networks, thus hindering users from using existing edge samplers 
to unify various random walk models in a scalable framework.
}

Therefore, to solve the diversity challenge and develop a unified random walk based NRL framework, 
we need an edge sampler that is able to efficiently sample from the unnormalized transition probability distributions.

\begin{table}[t]
	\centering
	\caption{The average acceptance ratio and time cost (sec.) when running node2vec with rejection edge sampler on Flickr.} 
    \label{tab:const}
	  \begin{tabular}{c|c|c|c|c|c}
	  \hline
	    $(p,q)$ & $(1,0.25)$ & $(1,4)$ & $(1,1)$ & $(4,1)$ & $(0.25,1)$\\
	    \hline
	    \hline
        Time(s) & 6.74 & 13.88& 6.08 & 6.21 & 15.81 \\
		AC Ratio& 0.86 & 0.36 & 1.00 & 0.99 & 0.25\\
        \hline
        Time Ratio& 1.11X & 2.28X & 1.0X & 1.02X & 2.60X\\
	  \hline
	  \end{tabular}%
\end{table}

In this paper, 
inspired by the Metropolis-Hastings (M-H) sampling technique~\cite{chib1995understanding}, which is a  Markov Chain Monte Carlo method for obtaining a sequence of random samples from a (unnormalized) probability distribution from which direct sampling is difficult,
we introduce an M-H based edge sampler for the random walk over networks.
The M-H based edge sampler chooses the next edge for a random walk via M-H sampling.
As the samples drawn by the M-H sampling technique are dependent,
the M-H based edge sampler might not be converged to the target probability distribution.
We theoretically show that when the uniform distribution is applied as the conditional probability mass function in the M-H based edge sampler,  
it can converge to \textbf{arbitrary} target probability distribution defined by the NRL models. 
 Benefit from the M-H sampling technique, our M-H based edge sampler can generate a sample in $\mathcal{O}(1)$ time complexity and $\mathcal{O}(1)$ memory complexity, and does not need to normalize the target probability distribution, thus solving the above two challenges perfectly. 

Although the M-H based edge sampler eventually converges to the target distribution, 
the initial samples may follow a very different distribution, especially if the starting point is in a region of low density. A common approach is to apply the burn-in period~\cite{gelmanbda04} for initialization, where an initial number of samples are thrown away.
However, the burn-in period incurs high initialization cost when handling large networks. 
As the alternative approaches, we propose two fast initialization strategies, called \textit{random} and \textit{high-weight}. 
{We theoretically analyze the trade-off between the two strategies and show that the high-weight strategy is more accurate than the random one when the ratio of maximal probability to minimal probability in the target distribution exceeds a certain threshold. On real-world networks, the transition probability distributions used by the NRL models satisfy the condition, therefore, the high-weight strategy leads to better results than the random one in practice.}

{To unify various existing random walk based NRL models, meanwhile, allow users to develop new NRL models,
we introduce a unified random walk model abstraction, in which a customized random walk model can be defined via specifying the \textit{dynamic edge weights} and the corresponding \textit{state}s of walkers. 
In other words, when defining a random walk model, users do not need to care about the normalized transition probability distribution. This unified abstraction is efficiently supported by our M-H based edge sampler.}

With the M-H edge sampler and the random walk model abstraction, 
we develop a scalable framework for random walk based NRL, named \sys. 
In the framework, to efficiently manage the massive M-H based edge samplers, 
we further design a 2D data layout, which allows us to query an M-H based edge sampler by walker states in constant time cost.
Finally, we test our \sys over eleven real-world datasets, including two billion-edge networks, with five considerable random walk based NRL models.
The empirical results demonstrate that \sys can be 10X-900X faster than the existing open-sourced versions.
{Compared to the state-of-the-art sampling techniques (e.g., KnightKing~\cite{yang2019knightking}, Memory-aware sampler~\cite{ma_sigmod2020}), our M-H based edge sampler achieves 9.6\%-73.2\% efficiency improvement in most of the settings and can scale to much larger networks.}

The contributions of this work are four-fold:
\begin{itemize}
    \item We propose a scalable framework for random walk based NRL models efficiently learning from billion-edge networks. The framework unifies different random walk models and allows users to develop new NRL models flexibly. 

    \item We introduce a new M-H based edge sampler for random walk models and theoretically analyze the correctness and efficiency of the sampler.
    
    \item We introduce two fast initialization strategies for the M-H based edge sampler and theoretically discuss the trade-offs between them.
    \item We comprehensively evaluate the performance of \sys and demonstrate the advantages of \sys over billion-edge networks.

\end{itemize}

%% file: 2.prel.tex
\section{Background}
In this section, we take a review of random walk based NRL models and introduce the Metropolis-Hastings sampling algorithm.
Table~\ref{tab:notation} lists the common notations in the paper.

\subsection{Popular Random Walk based NRL models}
\label{subsec:models}
We briefly review the five random walk based NRL models listed in Table~\ref{tab:existing_models} as representatives.

\textbf{Deepwalk.} \ It is the first random walk based NRL model. 
When generating random walks, in each step, it simply samples the next edge to visit based on the edge weights, and the transition probability distribution $\mathcal{G}_{\mathbf{x}}$ is 
\begin{equation}
    \mathcal{G}_{\mathbf{x}=v}(u)= \frac{w_{vu}}{\sum_{k \in \mathcal{N}(v)}w_{vk}},
    \label{eq:deepwalk_prob} 
\end{equation}
where the state $\mathbf{x}$ is defined as the current residing node $v$.

\textbf{Node2vec.} \ As a derivative of deepwalk, node2vec~\cite{grover2016node2vec} adopts the second-order random walk for preserving high-order structural information of networks. The walking procedure is guided by two hyper-parameters, $p$ and $q$, where $p$ restricts the tendency for the walker to maintain its position and search the local community, while $q$ controls the opportunity for exploring a wider scope. 
$\mathcal{G}_{\mathbf{x}}$ is defined as
\begin{equation}
\label{equ:node2vec}
\mathcal{G}_{\mathbf{x}=(s,v)}(u) = \frac{\alpha_u w_{vu}}{\sum_{k \in \mathcal{N}(v)}\alpha_k w_{vk}}, 
    \alpha_u = \left\{
\begin{array}{rl}
\frac{1}{p} & {d(u, s) = 0,}\\
1 & {d(u, s) = 1,}\\
\frac{1}{q} & {d(u, s) = 2,}\\
\end{array} \right.
\end{equation}
where $d(u, s)$ denotes the shortest unweighted distance between nodes $u$ and $s$. 
The state $\mathbf{x}$ of the walker in node2vec is defined as the previous edge $(s, v)$ the walker visited. 
\begin{table}[t]
    
    \centering
    \caption{Frequently Used Notations \label{tab:notation}}
     \linespread{1.1}
      \begin{tabular}{l|p{6cm}}
      
      \hline
        \textbf{Notations} & \textbf{Descriptions} \\ \hline
        $\mathcal{G}(\mathcal{V}, \mathcal{E})$ & A network with node set $\mathcal{V}$ and edge set $\mathcal{E}$\\ \hline
        $w_{e}$, $w_{uv}$ & Edge weight of an edge $e=(u,v)$ \\ \hline
        $w'_{e}$, $w'_{uv}$ & Dynamic edge weight of an edge $e=(u,v)$ defined by a random walk model \\ \hline
        $\mathcal{N}(v)$ & Set of the neighbor nodes for node $v$ \\ \hline
        $\mathbf{x}$, $\mathbf{x}_\lambda$ & The state for random walk process (of the specific walker $\lambda$). \\ \hline
        $\pi(\cdot)$ & A probability distribution \\ \hline
        $p, q$ & Parameters for node2vec and edge2vec models \\ \hline
        $\theta$ & Acceptance rate for sampling methods \\ \hline
        $\Phi(v), \Phi(v,u)$ & Type of node $v$ or edge $(v,u)$ in heterogeneous networks \\ \hline
        $T$ & A label in the heterogeneous network \\ \hline
        $\mathcal{G}_{\mathbf{x}}$ & The transition probability distribution under the consideration of the state $\mathbf{x}$ \\
      \hline
      \end{tabular}%
\end{table}

\textbf{Edge2vec.} \ Edge2vec~\cite{edge2vec} further extends the concept of node2vec to heterogeneous networks, 
and introduces the edge type transition matrix $\mathbf{M}$ allowing the random walk procedure to carry edge
semantic information.
$\mathbf{M}_{ij}$ means the probability for
the transition from the edge type $i$ to the edge type $j$. 
Therefore, $\mathcal{G}_{\mathbf{x}}$ in edge2vec becomes
\begin{equation}
\mathcal{G}_{\mathbf{x}=(s,v)}(u) = \frac{\alpha_u \mathbf{M}_{\Phi(s,v),\Phi(v,u)} w_{vu}}{\sum_{k\in \mathcal{N}(v)}\alpha_k \mathbf{M}_{\Phi(s,v),\Phi(v,k)} w_{vk}},
\end{equation}
where $\alpha_u$ follows Eq.~\ref{equ:node2vec}.  

\textbf{Matapath2vec~\cite{metapath2vec}.} It employs metapath-guided random walk for preserving semantics of heterogeneous networks. 
The random walks are restricted by a given metapath. 
The walker is only permitted to sample a next edge from the neighbors that match the next node type in the metapath. 
$\mathcal{G}_{\mathbf{x}}$ in metapath2vec can be formalized as below
\begin{equation}
\label{eq:metapath}
\mathcal{G}_{\mathbf{x}=(T,v)}(u) = 
    \left \{
    \begin{array}{ll}
     \frac{w_{vu}}{\sum_{k\in \mathcal{N}(v) \land \Phi(k) = T}w_{vk}}    &  \Phi(u) = T,\\
      0   &  \Phi(u) \ne T,
    \end{array}
    \right.
\end{equation}
where $T$ is the type of the next node for match in metapath. 
The state $\mathbf{x}$ of the walker in metapath2vec is $(T, v)$, including the current residing node $v$ and node type information. 

\textbf{Fairwalk.} The objective of fairwalk~\cite{fairwalk} is to learn network representations without the bias caused by the quantitative advantages of certain node attributes.
When sampling the next edge to visit, the model first selects a group of neighbors uniformly based on their types and selects the target edge by using node2vec model afterward. 
$\mathcal{G}_{\mathbf{x}}$ in fairwalk is
\begin{equation}
\label{eq:fairwalk}
    \mathcal{G}_{s,v}(u) = \frac{\alpha_u w_{vu}}{|\Phi|\sum_{k\in \mathcal{N}(v) \land \Phi(k)=\Phi(u) } \alpha_k w_{vk}}, 
\end{equation}
where $|\Phi|$ is the number of node types in $\mathcal{N}(v)$, and $\alpha_u$ follows the definition in Eq. \ref{equ:node2vec}.

\subsection{Metropolis-Hastings Sampling Algorithm}
\label{subsec:mh_algorithm}
Metropolis-Hastings (M-H) sampling algorithm is a sampling technique based on Markov Chain Monte Carlo (MCMC) methods. 
{Let $\Omega$ be the sample space, on which the target probability distribution $\pi(\cdot)$ is defined. The idea behind the MCMC method is to define a suitable Markov chain with $\Omega$ as its state space, such that $\pi(\cdot)$ is the unique stationary distribution of the chain. 
The M-H algorithm is a recipe for designing a Markov chain with a desired stationary
distribution $\pi(\cdot)$.} 
It starts by specifying a conditional probability mass function $q(y|x)$, $x,y\in \Omega$. 
For the $i^{\rm th}$ sampling, it draws a candidate sample $x^c$ according to the distribution $q(\cdot|x^{(i - 1)})$, where $x^{i-1}$ is the sample selected in the $(i-1)^{\rm th}$ sampling. 
Then, $x^c$ is accepted with the probability\\
\begin{equation}
    \theta(x^{i-1}) = \min \left\{ 1, \frac{\pi(x^c) \cdot q(x^{(i - 1)} | x^c))}
    {\pi(x^{(i - 1)}) \cdot q(x^c | x^{(i - 1)})} \right\}.
    \label{eq:orig_acceptance}
\end{equation}
If accepted, then $x^{(i)}$ is set to $x^c$. Otherwise, $x^{(i)}$ is set to the last sampling result (i.e., $x^{(i - 1)}$). 

{The conditional probability mass function $q(y|x)$ should be chosen so that the generated Markov chain is irreducible and aperiodic, which guarantees the convergence of the Markov chain~\cite{roberts1994simple}. Beyond that, $q(y|x)$ is a design parameter to obtain good performance. It governs both the rate of convergence of the Markov chain as well as the computational complexity of picking a proposed next state. }
The following theorem gives the condition of Kullback-Leibler (KL) divergence~\cite{Kullback51klDivergence} between the distribution $\pi^i$ and the target distribution $\pi$ to be {geometrically} converged for a Markov chain, where $\pi^{i}$ is the probability distribution on a Markov chain after $i$ iterations.
\begin{theorem} [Proposition 1 in \cite{chauveau2006selection}.]
    \label{lemm:convergence}
    If the conditional probability mass function $q(\cdot | \cdot)$ satisfies  $q(y | x) \geq a \pi(y) \text { for all } x, y \in \Omega$ and $a\in (0,1)$, then 
    \begin{equation}
        \label{eq:lem}
        \begin{split}  
        \mathrm{KL}\left(\pi^{i}, \pi\right) &\leq  \kappa \rho^{i}\left(1+\kappa \rho^{i}\right), \\ 
    \end{split}
    \end{equation}
    where 
    \begin{equation}
    \label{eq:kappa}
        \kappa= \left\|\pi^{0} / \pi-1\right\|_{\infty},  \rho=(1-a).
    \end{equation}
\end{theorem}
{Note that the notation $\left\|\cdot\right\|_{\infty}$ means the supremum norm. The coefficient $a$ is determined by the relation between the distributions $q(\cdot|\cdot)$ and $\pi(\cdot)$. According to Eq.~\ref{eq:kappa}, a larger $a$ leads to faster converging. This can help us intuitively compare the merits among different $q(\cdot|\cdot)$s.}

{As in any MCMC method, the draws of M-H sampling method are regarded as a sample from the target probability distribution $\pi(\cdot)$ only after the method has converged. However, it is still an empirical question of how many iterations are required before the convergence~\cite{chib1995understanding}. A common approach is to apply the \textbf{burn-in} period~\cite{gelmanbda04}, in which an initial number of samples are discarded, to initialize the Markov chain. We usually determine the number by setting it sufficiently large (e.g., 1000) or conducting parameter tuning.} 

%% file: 3.tech_new.tex
\section{Fast Random Walk with M-H Sampling}
\label{sec:ne_with_mh}
In this section, we introduce our new edge sampler by adopting the M-H sampling method, thus achieving fast random walk generation over large networks. 
\subsection{M-H Based Edge Sampler}
\label{subsec:sample}
The function of an edge sampler is to randomly draw edges from the target transition probability distribution under the consideration of walker state $\mathbf{x}$. 
Given a node $v$, we assume the unnormalized target distribution is $\{ w'_{vu} | u \in \mathcal{N}(v) \}$, where $w'_{vu}$ is the dynamic edge weight defined by the random walk model.
To apply the M-H sampling method, for all $u\in \mathcal{N}(v)$, we set the conditional probability mass function $q(\cdot|u)$ to be uniform distribution, i.e., $q_v(\cdot | u) = 1 / deg(v)$, where $deg(v)$ is the degree of the node $v$. 
Because of the symmetry of the uniform distribution (i.e., $q_v(s|u) = q_v(u|s)$, $s,u\in \mathcal{N}(v)$), the acceptance ratio $\theta$ (Eq.~\ref{eq:orig_acceptance}) for the $i^{\rm th}$ candidate $x^c$ can be reduced to
$\theta = \min \left\{ 1, \frac{w'_{vx^c}}{w'_{vx^{(i - 1)}}} \right\}$.

\begin{algorithm}[t]
    \caption{\small M-H based edge sampler}
    \label{alg:sampler}
    \small
    \begin{algorithmic}[1]
    \Require
        Network $\mathcal{G}$, random walk model $M$, walker state $\mathbf{x}$ \;
    \Ensure 
        The sampled next edge\;
    \State Current node $v \gets \mathbf{x}.getCurrentNode()$\;
    \State Draw an edge $(v,u)$ uniformly from $ \{ (v,u) | u \in \mathcal{N}(v) \} $\;
    \State $w_{vu}'\gets M.\textsc{CalculateWeight}(\mathbf{x}, (v,u))$ \;
    \State $s \leftarrow$ $\textsc{Last}_{\mathbf{x}}$, which is the last sampled node.
    \State $w_{vs}' \gets M.\textsc{CalculateWeight}(\mathbf{x}, (v,s))$ \;
    \State Calculate acceptance ratio $ \theta \gets \min\left\{ 1, \frac{w_{vu}'}{w_{vs}'} \right\} $\;
    \If {rand(0,1) $< \theta$} 
        \State $\textsc{Last}_{\mathbf{x}} \gets u$\; 
    \EndIf
    \State \Return $(v, \textsc{Last}_{\mathbf{x}})$ \;
    \end{algorithmic}
\end{algorithm}

Algorithm~\ref{alg:sampler} gives the procedures of the M-H based edge sampler. \textsc{CalculateWeight} method refers to get the dynamic edge weight based on the current walker state $\mathbf{x}$, while $\textsc{Last}_\mathbf{x}$ refers to the variable preserving the last sampling result of the sampler corresponding to state $\mathbf{x}$. 
If the candidate sample is accepted according to the acceptance ratio, $\textsc{Last}_\mathbf{x}$ is updated. Otherwise, $\textsc{Last}_\mathbf{x}$ is repeatedly used.

By considering the influence of a conditional probability mass function $q(\cdot|u)$ introduced in Section~\ref{subsec:mh_algorithm}, we adopt the uniform distribution for two reasons. First, the uniform distribution is efficient for the calculation of the acceptance ratio and the selection of candidate; Second, uniform distribution is independent of the target distribution which helps us develop a unified NRL framework introduced in Section~\ref{sec:sys}.

\textbf{Complexity Analysis.} 
Since both the time and memory complexity of M-H sampling method is $\mathcal{O}(1)$, the complexity of the edge sampler is dominated by the calculation of the new weights (Lines 3, 5 in Algorithm~\ref{alg:sampler}). 
Taking node2vec as an example, the dynamic edge weight calculation takes $\mathcal{O}(\log (deg))$ in time complexity on average with binary search. 
Given a network $\mathcal{G}(\mathcal{V}, \mathcal{E})$, 
the M-H based edge sampler only entails $\mathcal{O}(\#state)$ memory complexity (e.g., $\mathcal{O}(|\mathcal{E}|)$ in node2vec), since each sampler only maintains one variable $\textsc{Last}_\mathbf{x}$.
In summary, the M-H based edge sampler is both efficient in time and memory aspects, thus supporting to handle large networks. 

\subsection{Convergence of M-H based Edge Sampler} 
In this subsection, we show that our edge sampler which uses uniform distribution as the conditional probability mass function can converge to any discrete target distributions. 

Before presenting the theorem, we introduce a lemma. 

\begin{lemma}
\label{lemma:uniform_dist_bound}
For arbitrary discrete probability distribution $\pi(\cdot)$, whose sample space $\Omega$ has size $n$, then $\pi_{max} \ge \frac{1}{n}$, where $\pi_{max}$ is the maximal probability of $\pi(\cdot)$.

\end{lemma}
\begin{proof}
Assuming that there exists a distribution where $\pi_{max} < \frac{1}{n}$, we have $\pi(x) < \frac{1}{n}, \ for \ all \ x \in \Omega$.
By accumulating all the probabilities, we have $\sum_{x\in \Omega}\pi(x) < 1$, which causes a contradiction. Therefore, the lemma holds.
\end{proof}

\begin{theorem}
    \label{theo:convergence_condition}
    When the M-H based edge sampler uses uniform distribution as the conditional probability mass function, for any discrete target distribution $\pi$, there exists a constant $a\in(0,1)$ such that $q(y | x) \geq a \pi(y) \text { for all } x, y \in \Omega$. 
\end{theorem}

\begin{proof}
    The uniform distribution in edge samplers means
    \begin{equation}
        q(y|x) = \frac{1}{deg} \ for \ all \ x,y \in \Omega,
    \end{equation}
    where $deg$ denotes the degree of the current residing node. 
    To satisfy the constraint $q(y | x) \geq a \pi(y)$, $a$ must be
    \begin{equation}
        a \le \frac{1}{deg\cdot\pi(y)} \ for \ all \ y \in \Omega.
    \end{equation}
    Then we can set
    \begin{equation}
        a = \frac{1}{deg \cdot \pi_{\mathrm{max}}},
    \end{equation}
    and on the basis of Lemma~\ref{lemma:uniform_dist_bound}, we have $a\in(0,1)$.
\end{proof}

Theorem~\ref{lemm:convergence} and \ref{theo:convergence_condition} guarantee the correctness of our M-H based edge sampler for any user defined random walk model over networks. Note that we ignore the case $a=1$ in Theorem~\ref{theo:convergence_condition}. This is because the target distribution is also a uniform distribution when $a=1$, and then the divergence $\mathrm{KL}(\pi^{i}, \pi)=0$ directly.

\subsection{Initialization of M-H based Edge Sampler}
\label{subsec:init}
{As introduced in Section~\ref{subsec:mh_algorithm}, to initialize an M-H sampling method, the common approach -- \textbf{burn-in} strategy needs to discard a certain number of initial samples.  
With regard to the M-H based edge sampler, the burn-in strategy incurs an expensive cost for the initialization of $\#state$ edge samplers over a network.
In the following, we introduce two new efficient initialization strategies.}

\subsubsection*{\textbf{Random Initialization}} To avoid the expensive initialization cost, a basic solution is to draw an initial sample $x^{(1)}$ from a uniform distribution, i.e., $\pi^0(\cdot) = 1/n$, called \textit{random initialization} strategy. 
It is clear that this initialization strategy guarantees efficiency (i.e., time complexity $\mathcal{O}(1)$), even when a large number of samplers are in need of initialization. 
However, the random initialization cannot guarantee the convergence of the sampler in the beginning, thus may lead to the issue of imprecision of the statistics of the generated random walks.

\subsubsection*{\textbf{High-weight Initialization}} Recall the burn-in strategy, the objective of discarding some initial samples is to make sure the samples enter a high probability region, a place where the states of the Markov chain are more representative of the distribution we are sampling from. 
This is because in practice we always stop after some finite $N$ iterations. If $N$ is small and the generated samples are in a low probability region, the distribution of the samples is a poor approximation to the target distribution. 

With such observation, we propose a new initialization strategy, called \textit{high-weight initialization}.
Specifically, when initializing a sampler, it selects the edge with maximum weight as the initial sample, which possesses the maximal probability (i.e., high probability region) in the target distribution. To get the exact maximal weight, for each state, we need to traverse all neighbor edges, and this entails $\mathcal{O}(deg)$ time complexity.
To speed up the selection, we turn to get an approximate maximal weight by uniformly sampling a portion of neighbors and reserving the edge with the maximal weight. 
The law of large numbers~\cite{grimmett2001probability} guarantees that we can obtain the exact maximal probability when
the sample size is large enough.

\begin{figure*}[t]
     \centering
     \subfigure[$n=10$]{
     \label{mi-b1}
     \begin{minipage}[t]{0.25\linewidth}
     \centering
     \includegraphics[width=1.8in]{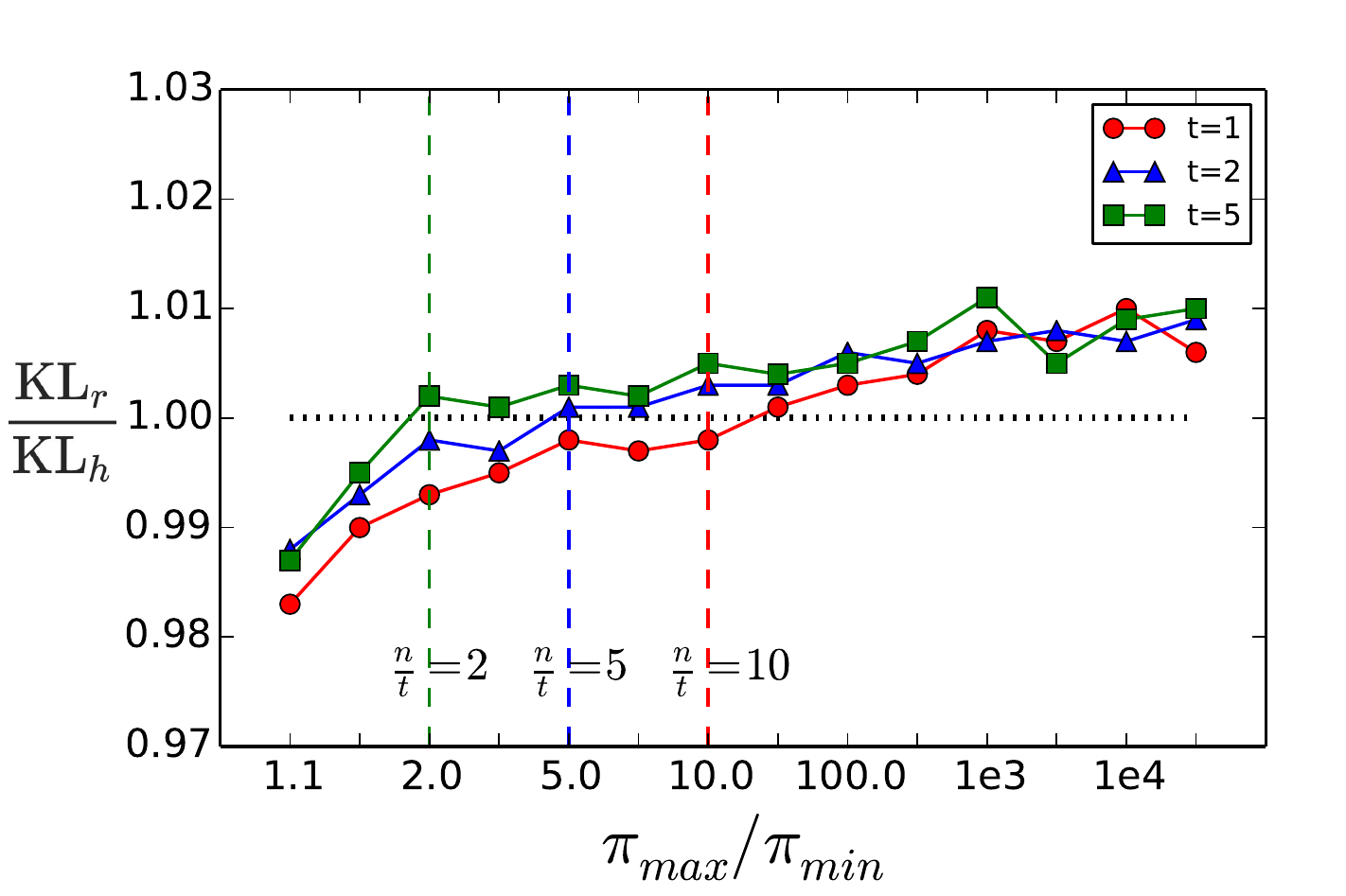}
     \end{minipage}%
     }%
     \subfigure[$n=100$]{
     \label{mi-b2}
     \begin{minipage}[t]{0.23\linewidth}
     \centering
     \includegraphics[width=1.8in]{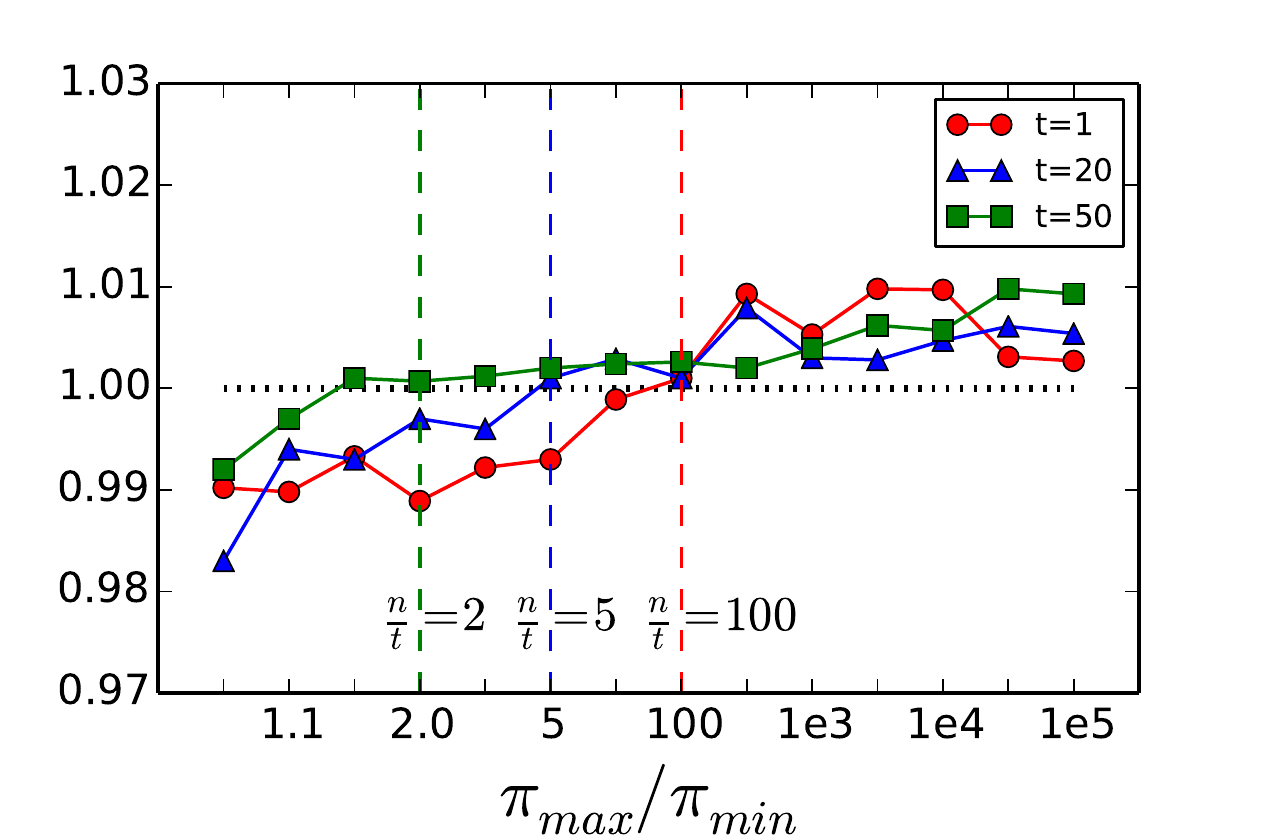}
     \end{minipage}%
     }%
    \subfigure[$n=1000$]{
    \scalebox{0.35}{
    \includegraphics{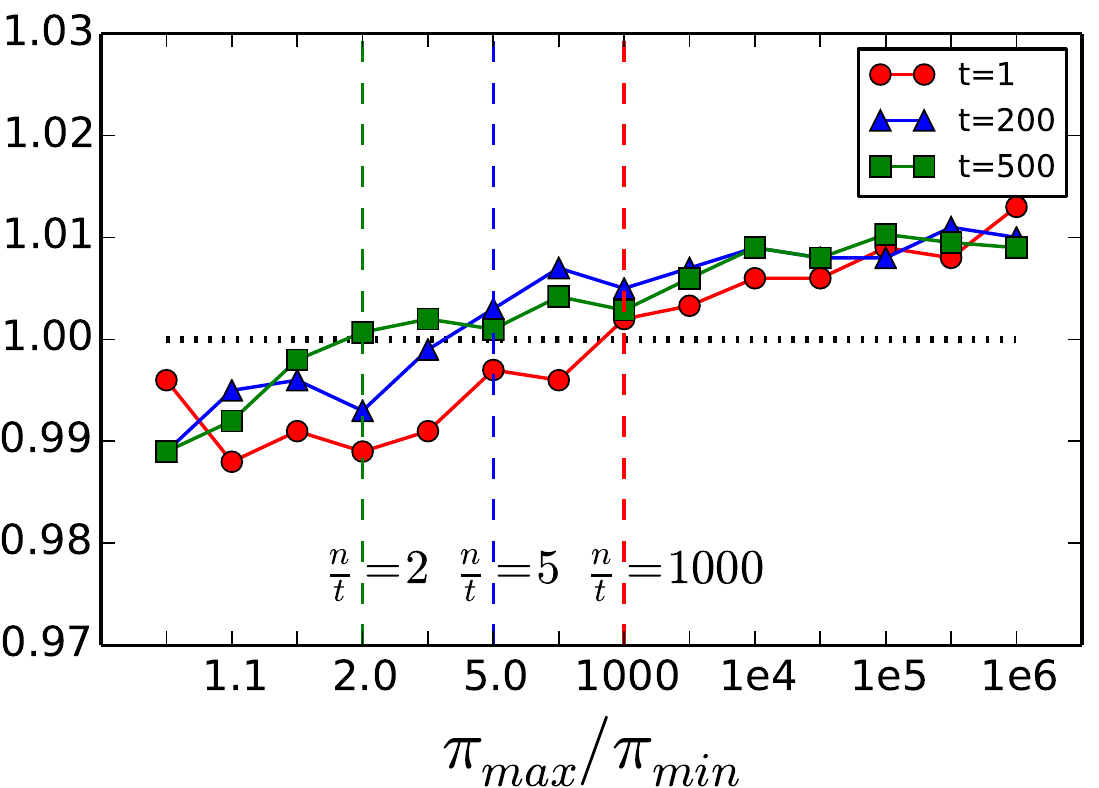}
    }
    }%
    \subfigure[$n=10000$]{
    \scalebox{0.35}{
        \includegraphics{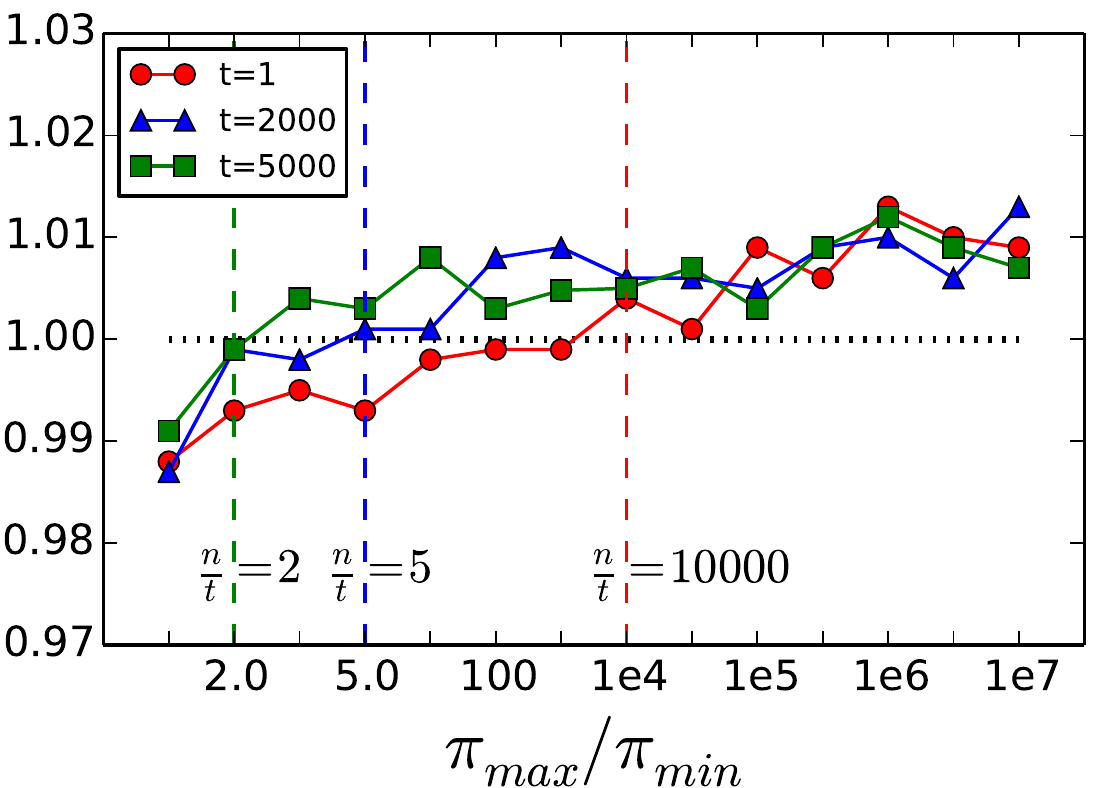}
    }
    }%
    \caption{\small The ratio of KL divergences with random initialization and high-weight initialization. $\mathrm{KL}_r$ ($\mathrm{KL}_h$) is the KL divergence between the target distribution and the sampled distribution with random (high-weight) initialization.  \label{fig:sampler_init_expr}}
    \end{figure*}
\textbf{Theoretical Trade-off Analysis.}
According to Theorem~\ref{lemm:convergence}, the convergence speed is not only affected by $\rho$, but also influenced by the difference between initial probability distribution $\pi^0$ and the target distribution $\pi$, i.e., the coefficient $\kappa$ in Eq. \ref{eq:kappa}. 
Hence, in terms of the convergence speed, random initialization and high-weight initialization may surpass each other with respect to different target distributions.

The following theorem gives the condition when high-weight initialization works better than random initialization.
\begin{theorem}
    \label{theo:high_prob}
    The high-weight initialization strategy leads to a faster convergence than the random one, 
    when the target distribution $\pi$ satisfies the following conditions:
    
    \begin{equation} \label{eq:init_cond}
        \left\{
            \begin{array}{l}
                \pi_{\mathrm{max}} < \frac{1}{2t} \\
                \frac{\pi_{\mathrm{max}}}{\pi_{\mathrm{min}}} > \frac{n}{t} 
            \end{array}
        \right.
        or \
        \left\{
            \begin{array}{l}
                \pi_{\mathrm{max}} \ge \frac{1}{2t} \\
                \pi_{\mathrm{min}} < \frac{1}{2n} 
            \end{array}
        \right. 
    \end{equation}
    where $n$ is the size of sample space, $\pi_{max}$ and $\pi_{min}$ are the maximal and minimal probability in distribution $\pi$,
    $t$ is the number of elements whose probability are $\pi_{max}$.
\end{theorem}
\begin{proof}
The complete proof are presented in Appendix~\ref{app:proof}.
\end{proof}

\textbf{Empirical Comparison.} We conduct a set of simulations to empirically compare the random and high-weight initialization strategies. 
We first randomly generate 1000 target distributions for each configuration, 
which contains the sample space size $n$, the number of elements $t$ who have the maximal probability and the ratio $\frac{\pi_{\mathrm{max}}}{\pi_{\mathrm{min}}}$. 
For each target distribution, M-H based edge sampler with one initialization strategy conducts a sampling process by generating $5n$ samples and obtains an approximated distribution. This process is repeated 20 times.
After all the 1000 target distributions simulated by the M-H based edge sampler, we calculate the averaged KL divergence between the approximated distribution and the target distribution. 

Fig.~\ref{fig:sampler_init_expr} presents the ratio of KL divergences of samplers adopting random initialization and high-weight initialization, i.e., $\frac{\mathrm{KL}_r}{\mathrm{KL}_h}$. If the ratio exceeds 1, the edge sampler with the high-weight initialization strategy generates samples more accurately, otherwise, the random initialization outperforms the high-weight initialization. From the figure, we clearly see that the empirical results are consistent with Theorem~\ref{theo:high_prob}. For example, when $n=1000$ and $t=200$, the ratio of KL divergence equals one when $\frac{\pi_{\mathrm{max}}}{\pi_{\mathrm{min}}} \approx \frac{n}{t}=5$. Meanwhile, when the ratio $\frac{\pi_{\mathrm{max}}}{\pi_{\mathrm{min}}} > \frac{n}{t}$, which entails a skew distribution, the accuracy of high-weight initialization trumps random initialization, which appears in the figure as the ratio exceeds 1.

%% file: 4.sys.tex
\section{\sys: A Unified Framework for Scalable Network Representation Learning}
\label{sec:sys}
In this section, we introduce \sys, which is a unified random walk based NRL framework for scalable training.
\sys relies on the M-H based edge sampler to improve the efficiency of model training. 

\subsection{Overview of \sys}
Figure~\ref{fig:architecture} illustrates the overview of \sys. 
Given a network $\mathcal{G}$, \sys loads it into memory and generates random walks according to the NRL models. 
In \sys, we generalize the process of generating random walks as the life cycle of a \textit{walker}. {All the M-H based edge samplers are maintained in the sampler manager, and each edge sampler executes its initialization strategy only when it is first called.
The overall procedure of random walk generation in \sys is shown in Algorithm~\ref{alg:walks}. 
An instance of walker $\lambda$ is created whenever a walking procedure initiates by calling the function $\textsc{GetWalker}$ in Line 5. 
For each iteration of walking step (Lines 8-13), the walker $\lambda$ first uses the function $\textsc{WalkerStates}$ to get its state $\mathbf{x}$, then queries the edge sampler from the sampler manager by the state $\mathbf{x}$ via the function $\textsc{QuerySampler}$.
After the sampling result is calculated through Algorithm~\ref{alg:sampler}, the walker state is updated using the function $\textsc{UpdateState}$.}
The walking process will not stop until the walker reaches the maximal steps $L$ defined by users.
In addition, walkers starting from different initial nodes are able to run simultaneously in \sys because of the independence between walkers. 
Indeed, \sys parallelizes the random walk generation by assigning walkers to threads evenly. 
In the learning phase, all the generated walks are fed into the trainer module as the training corpus, and \sys executes SGD optimization over Word2Vec models. This phase can be parallelized with the techniques~\cite{8663393} proposed in ML community as well.

\begin{figure}[t]
    \centering
    \includegraphics[width=0.42\textwidth]{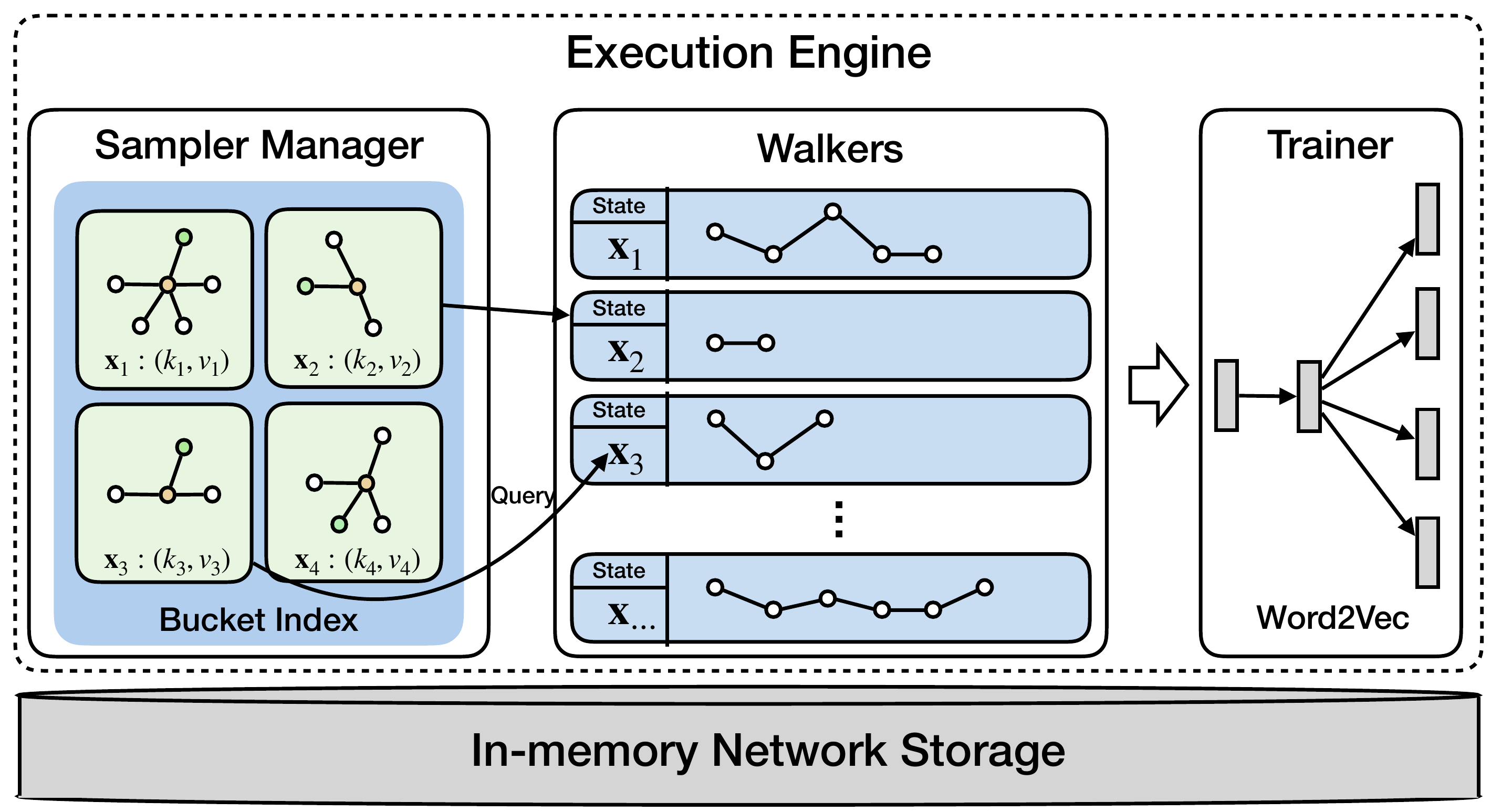}
    \caption{The overview of \sys.}
    \label{fig:architecture}
\end{figure}

\begin{algorithm}[t]
    \caption{\small The procedures of random walk generation}
    \label{alg:walks}
    \small
    \begin{algorithmic}[1]
    \Require
        Network $\mathcal{G}(\mathcal{V,E})$, random walk model $M$, walk length $L$, number of walks $K$
    \Ensure 
        Random walk sequences $W$
    \State \# parallel
    \For{$v$ in $V$}
    \State \# parallel
    \For{$k$ in 1 $\cdots K$}
        \State $\lambda \gets \textsc{GetWalker}(M, v)$
        \State append $v$ to $S$
        \For{$step$ in 1 $\cdots$ $L$}
            \State {$\mathbf{x} \gets \textsc{WalkerStates}(\lambda)$}
            \State {$edgeSampler \gets \textsc{QuerySampler}(\mathbf{x})$ }
            \State {{\# Alg.~\ref{alg:sampler} defines the sampling logic.}}
            \State {{$next \gets edgeSampler.\textsc{Sampling}(\mathcal{G}, M, \mathbf{x})$}}
            \State {$\mathbf{x} \gets \textsc{UpdateState}(\mathbf{x}, next)$}
            \State append $next$ to $S$
        
        \EndFor
        \State insert $S$ into $W$
    \EndFor
    \EndFor
    \State \Return $W$
    \end{algorithmic}
\end{algorithm}
\subsection{Unified Random Walk Model Abstraction}
To make \sys easy for users to program various random walk based NRL models, 
we introduce a unified random walk model abstraction on top of the M-H based edge sampler. 

The difference among random walk models lies in the transition probability. 
Recall that the transition probability corresponding to a state $\mathbf{x}$ can be expressed as a normalized form:
\begin{equation}
    \mathcal{G}_{\mathbf{x}}(u) = \frac{w'_{\mathbf{x}u}}{\sum_{k \in \mathcal{N}(v)} w'_{\mathbf{x}k}},
    \label{eq:general_weight_model}
\end{equation}
where $w'$ is the \textit{dynamic edge weight}, which actually determines the distribution and is related to the state $\mathbf{x}$.
In other words, each state is associated with a specific transition probability distribution. 
As a consequence, to define a random walk, we only need to specify the state and the corresponding dynamic edge weight.
Table~\ref{tab:def} lists the dynamic edge weight $w'$ and the state definition for the previously introduced five NRL models. 

\begin{table}[t]
\centering
    \centering
    \caption{State definition and dynamic edge weights. \label{tab:def}}
    \scalebox{1}{
      \begin{tabular}{c|c|c}
      \hline
      Model    & State $\mathbf{x}$  & Dynamic Weight of Edge $(v,u)$ \\ \hline
      deepwalk & $v$              & $w_{vu}$         \\ \hline
      node2vec & $(s,v)$          & $\alpha w_{vu}$          \\ 
      \hline
      edge2vec & $(s,v)$          & $\alpha \mathbf{M}_{(s,v),(v,u)} w_{vu}$   \\ 
      \hline
      fairwalk & $(s,v)$                 &  $ \frac{\alpha w_{vu}}{|K|}$, ($k \in K \Leftrightarrow \Phi(k) = \Phi(u)$)  \\ 
      \hline
      metapath2vec& $(T,v)$ &  
      $\left\{
        \begin{aligned}  
        w_{vu}&, \ if \ \Phi(u) = T \\
        0&, \ otherwise
        \end{aligned}
        \right.$\\ 
      \hline
      \end{tabular}%
      }
\end{table}

Benefiting from the M-H based edge sampler, which is able to directly sample the unnormalized probability distributions, 
the unified abstraction above can be efficiently supported in \sys over large networks by computing the dynamic edge weight online.
{
We provide two key programming interfaces in \sys to realize the unified abstraction, 
named \textsc{CalculateWeight} and \textsc{UpdateState}. 
With \textsc{CalculateWeight}, the dynamic edge weight can be computed based on the current state and the next candidate edge, and \textsc{UpdateState} specifies the update logic of states given the next edge. 
Therefore, to define a random walk based NRL model in \sys, users only need to implement the above two interfaces. Figure \ref{fig:api} gives the implementation for node2vec following the interfaces provided by \sys.
}

\subsection{Implementation Details}

\textbf{Network Storage.} \ In order to satisfy the requirement of storing large networks in memory, 
we utilize compressed sparse row (CSR) as the basic format of data storage, which consists of a node list and an edge list for an unweighted network. 
For each node, CSR keeps the offset in the edge list corresponding to its neighboring edges. 
CSR is also efficient for the storage of weighted networks by allocating an additional weight value for each entry in the edge list. 
In addition to that, to store the heterogeneous networks, we keep the types of each node by allocating an extra array of size $|\mathcal{V}|$.

\begin{figure}[t]
    \centering
    \includegraphics[width=0.5\textwidth]{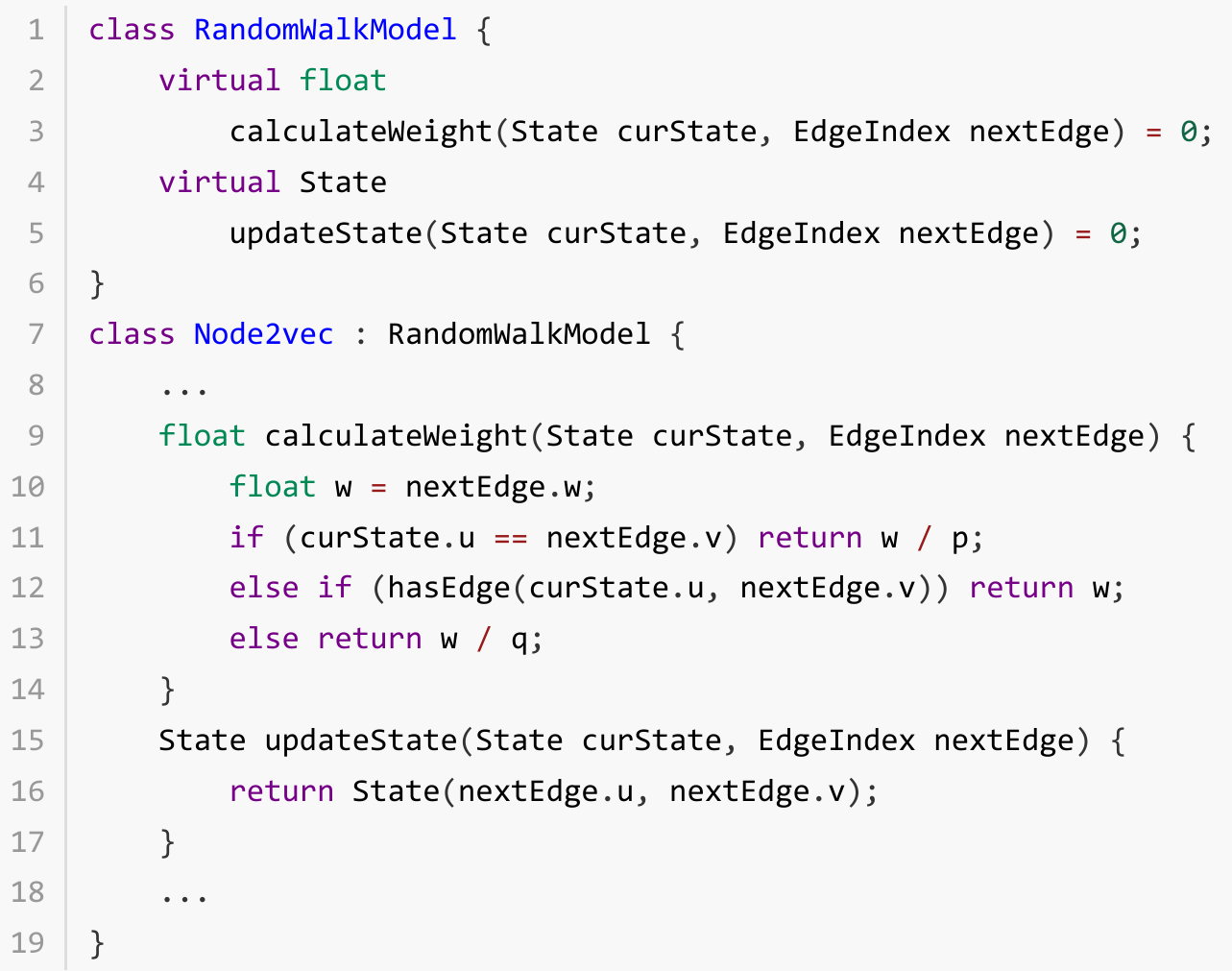}
    \caption{Programming interfaces for defining a random walk model and the example implementation for node2vec.}
    \label{fig:api}
\end{figure}
\begin{figure}[t]
    \centering
    \includegraphics[width=0.32\textwidth]{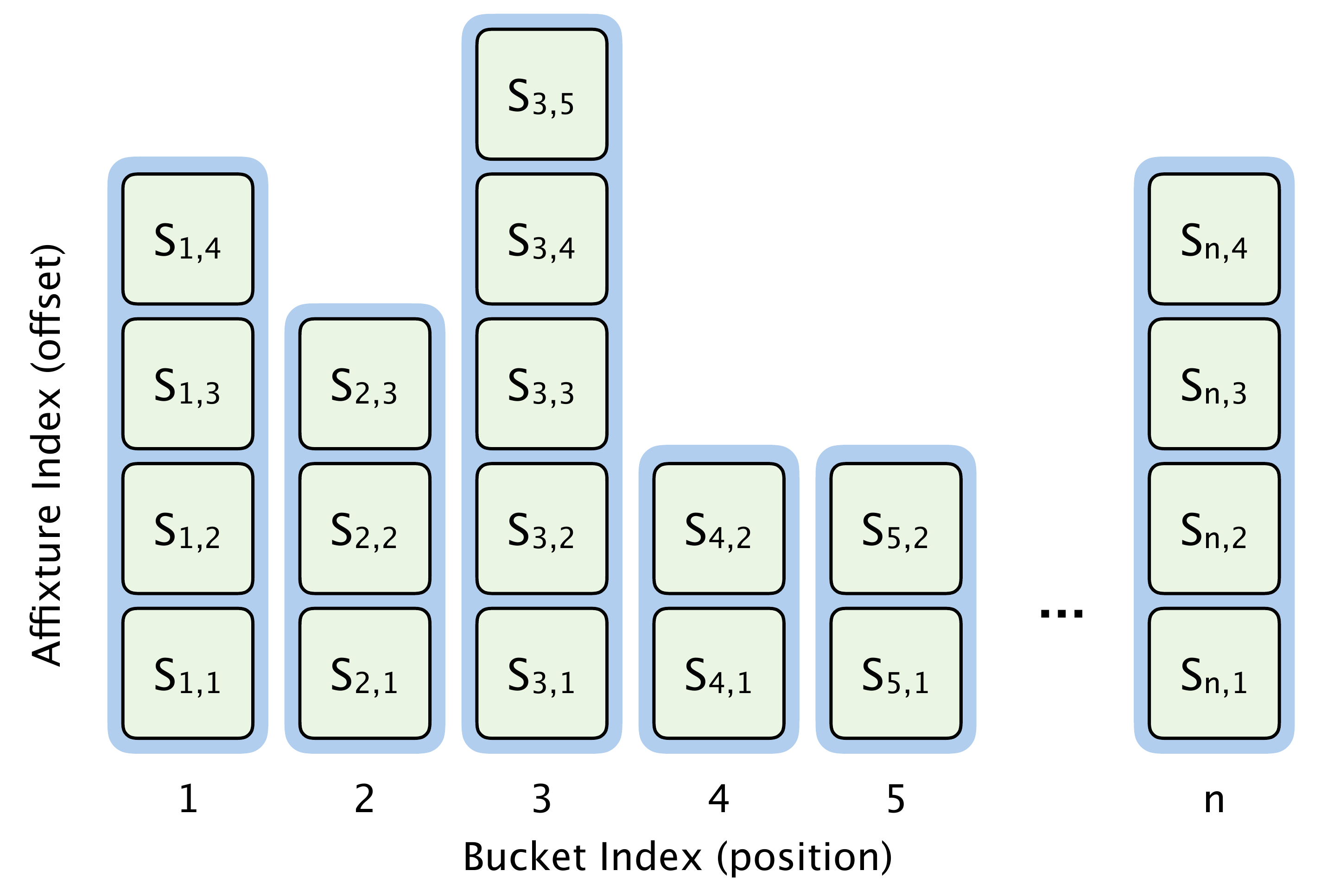}
    \caption{2D data layout for the samplers in node2vec.}
    \label{fig:2d_data_layout}
\end{figure}
\begin{figure*}[t]
    \centering
    \subfigure[BlogCatalog Micro-F1 Score]{
     \begin{minipage}[t]{0.24\linewidth}
    \centering
    \includegraphics[width=1.9in]{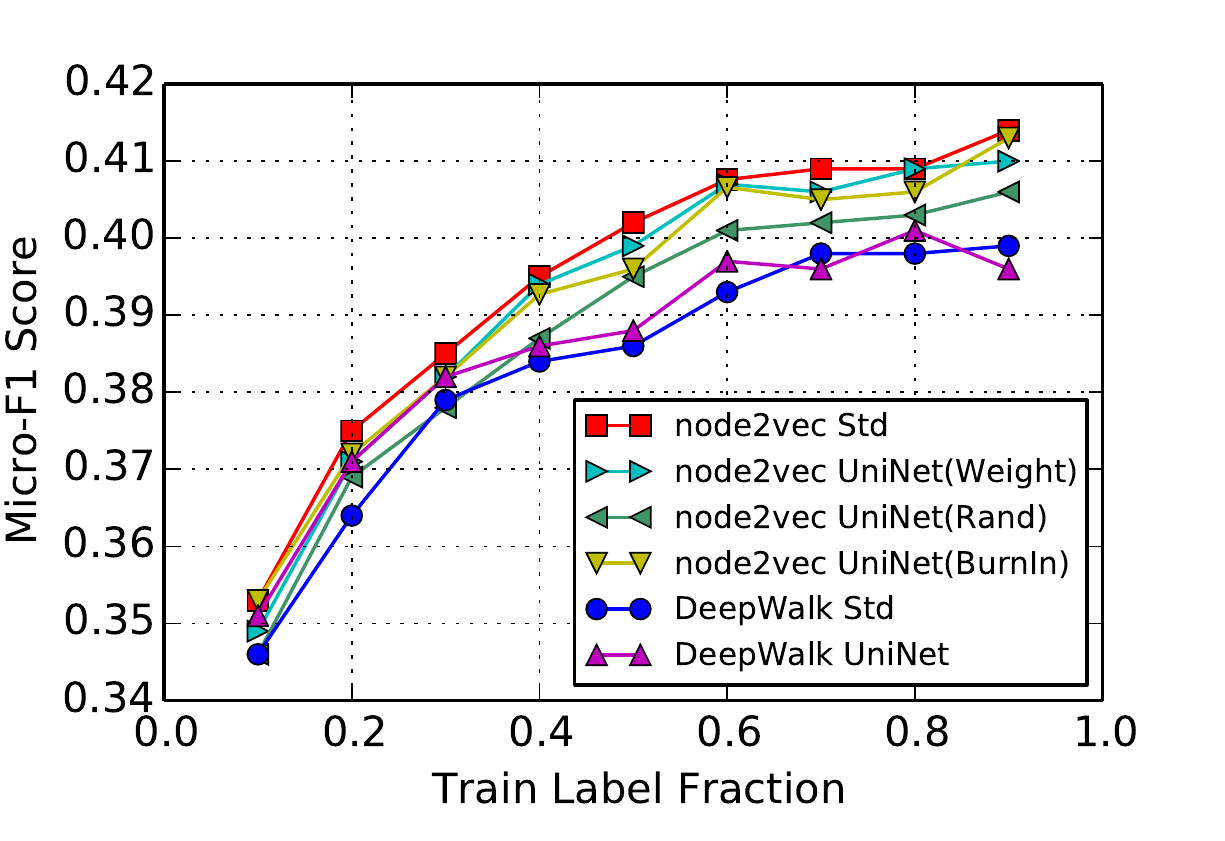}
    \label{subfig:bc_micro}
     \end{minipage}%
    }%
    \subfigure[Flickr Micro-F1 Score]{
     \begin{minipage}[t]{0.24\linewidth}
    \centering
    \includegraphics[width=1.9in]{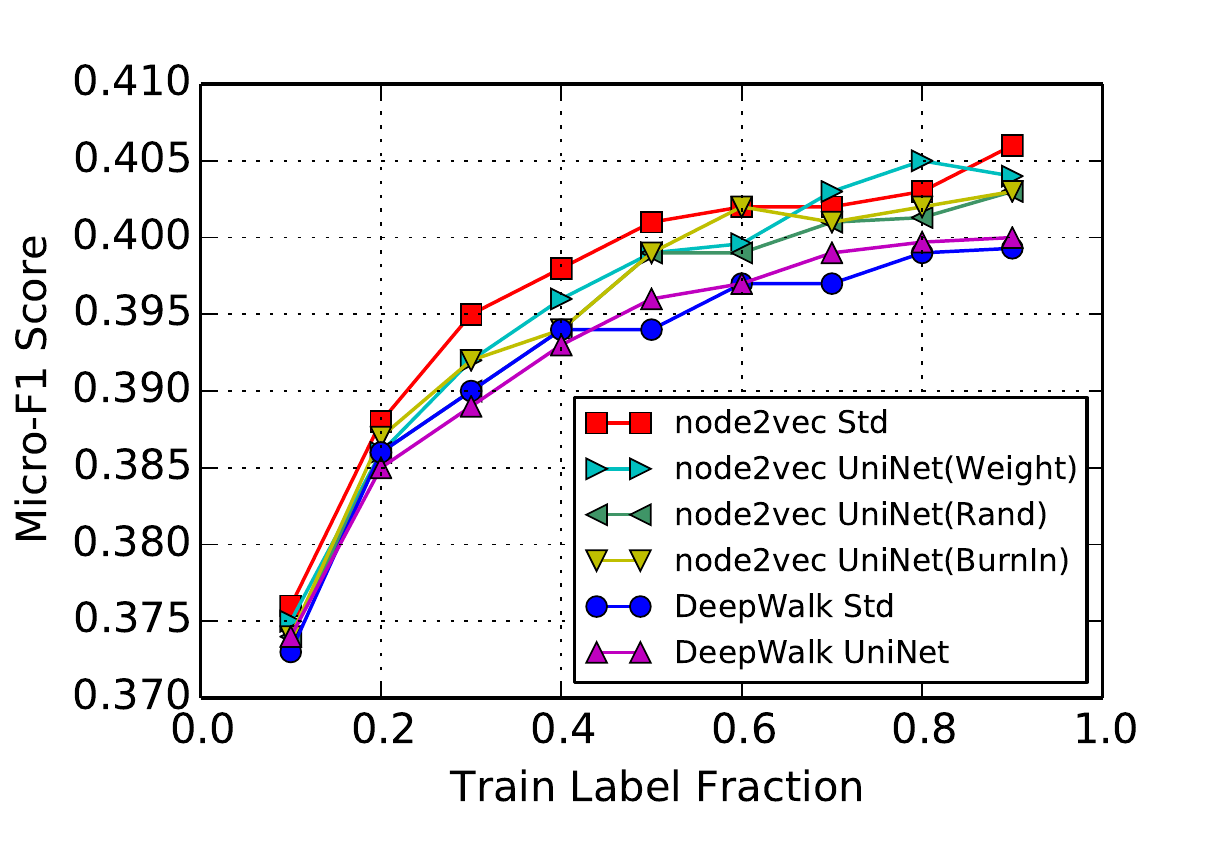}
    \label{subfig:flickr_micro}
     \end{minipage}%
    }%
    \subfigure[{Reddit Micro-F1 Score}]{
     \begin{minipage}[t]{0.24\linewidth}
    \centering
    \includegraphics[width=1.9in]{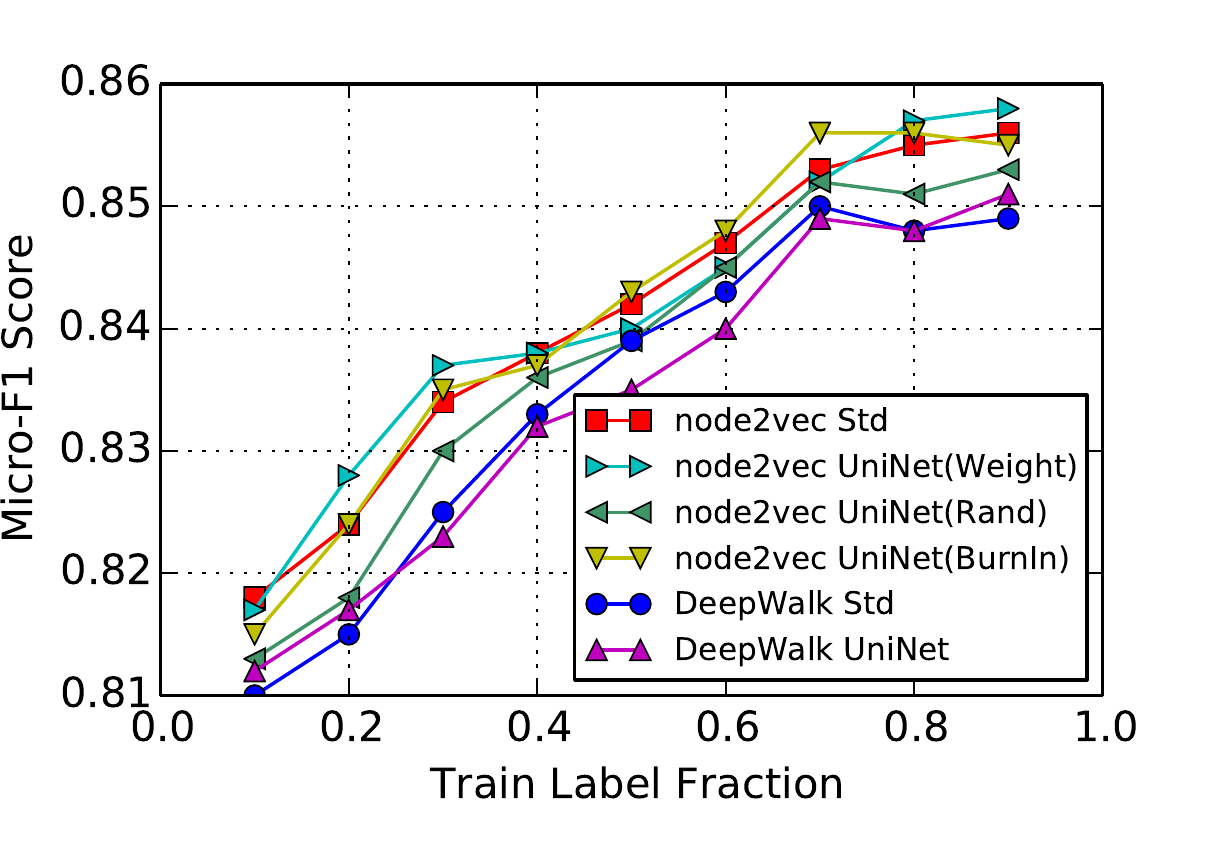}
    \label{subfig:reddit_micro}
     \end{minipage}%
    }%
    \subfigure[Aminer Micro-F1 Score]{
     \begin{minipage}[t]{0.24\linewidth}
    \centering
    \includegraphics[width=1.9in]{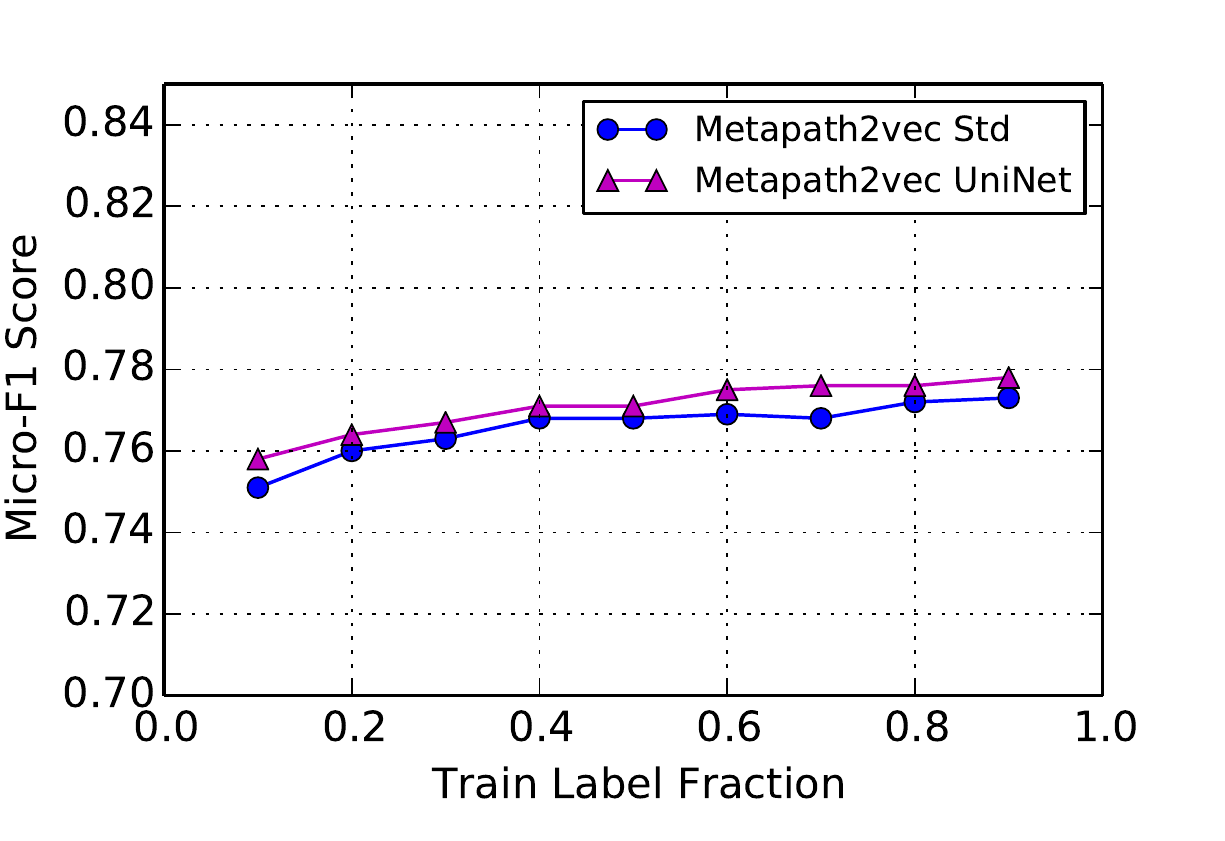}
    \label{subfig:aminer_micro}
     \end{minipage}%
    }%
    
    \subfigure[BlogCatalog Macro-F1 Score]{
     \begin{minipage}[t]{0.24\linewidth}
    \centering
    \includegraphics[width=1.9in]{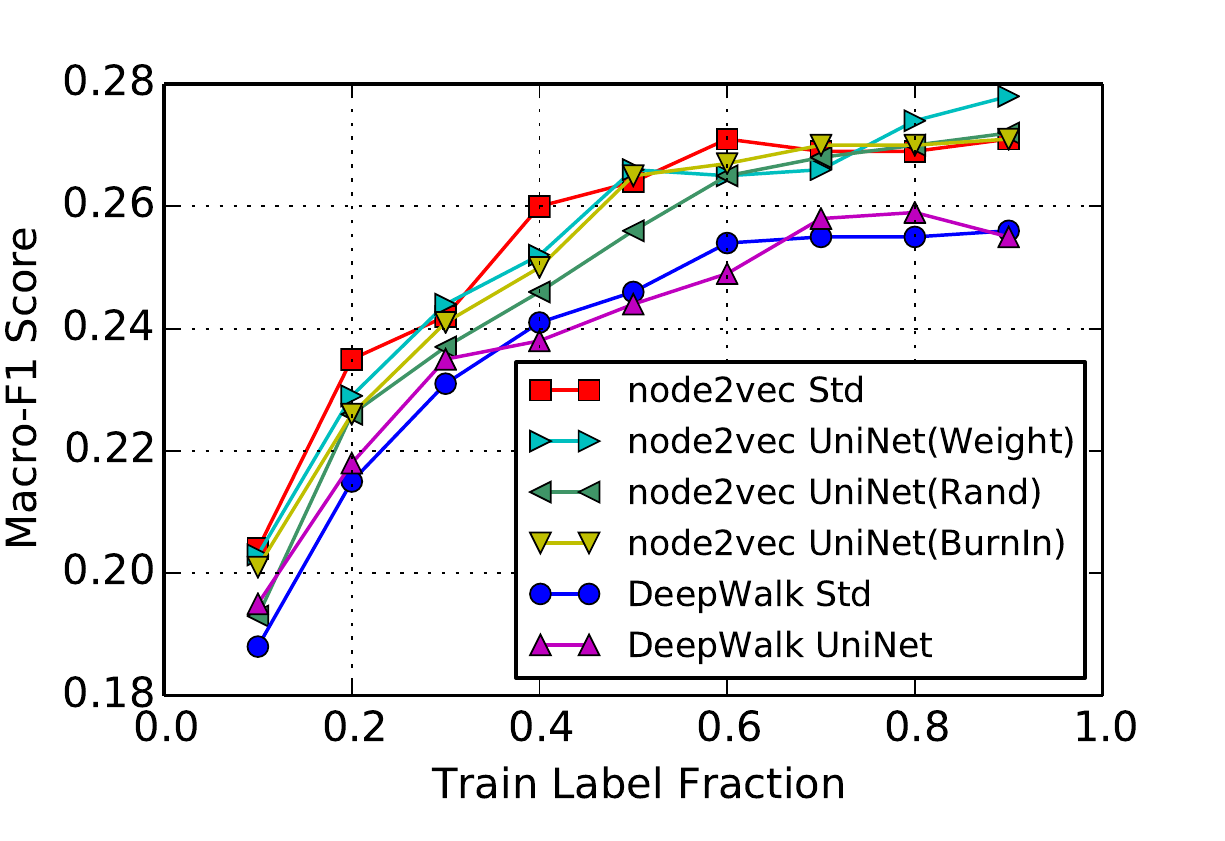}
    \label{subfig:bc_macro}
     \end{minipage}%
    }%
    \subfigure[Flickr Macro-F1 Score]{
     \begin{minipage}[t]{0.24\linewidth}
    \centering
    \includegraphics[width=1.9in]{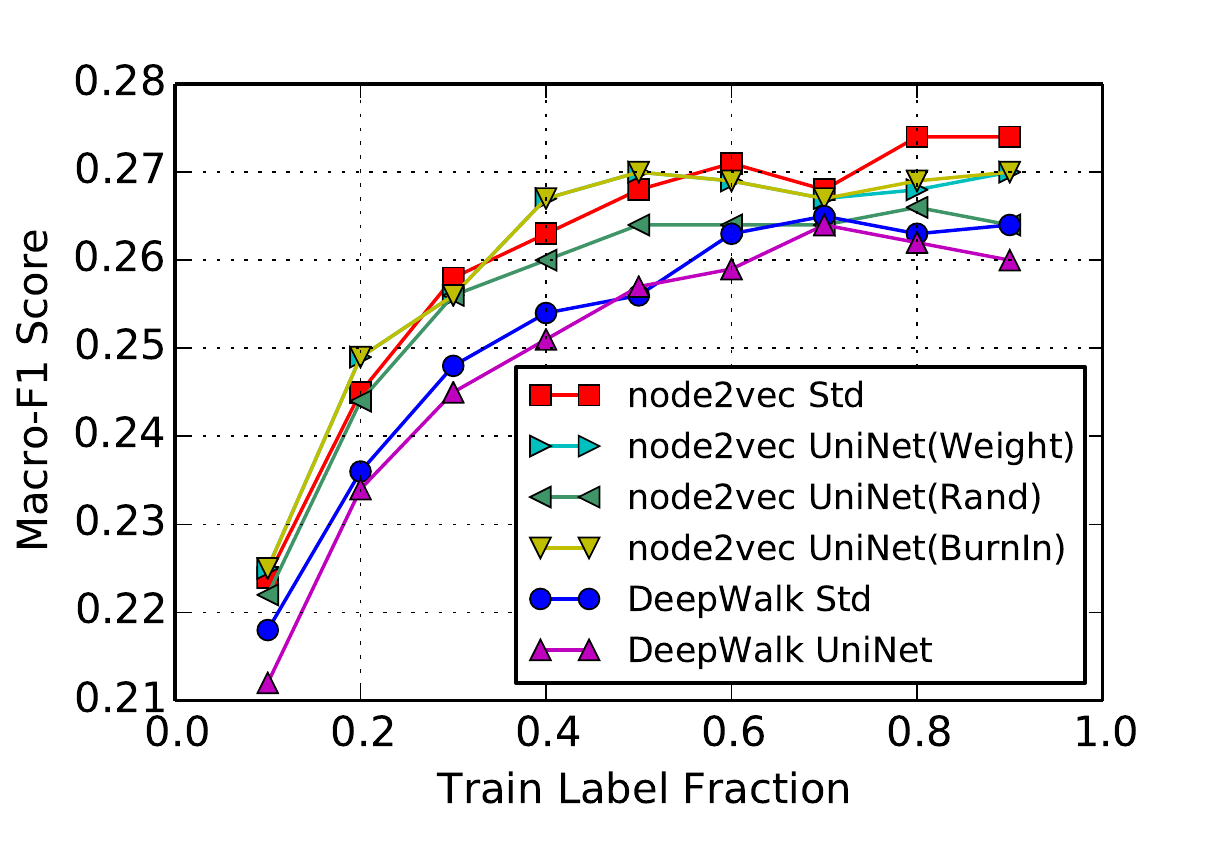}
    \label{subfig:flickr_macro}
     \end{minipage}%
    }%
    \subfigure[{Reddit Macro-F1 Score}]{
     \begin{minipage}[t]{0.24\linewidth}
    \centering
    \includegraphics[width=1.9in]{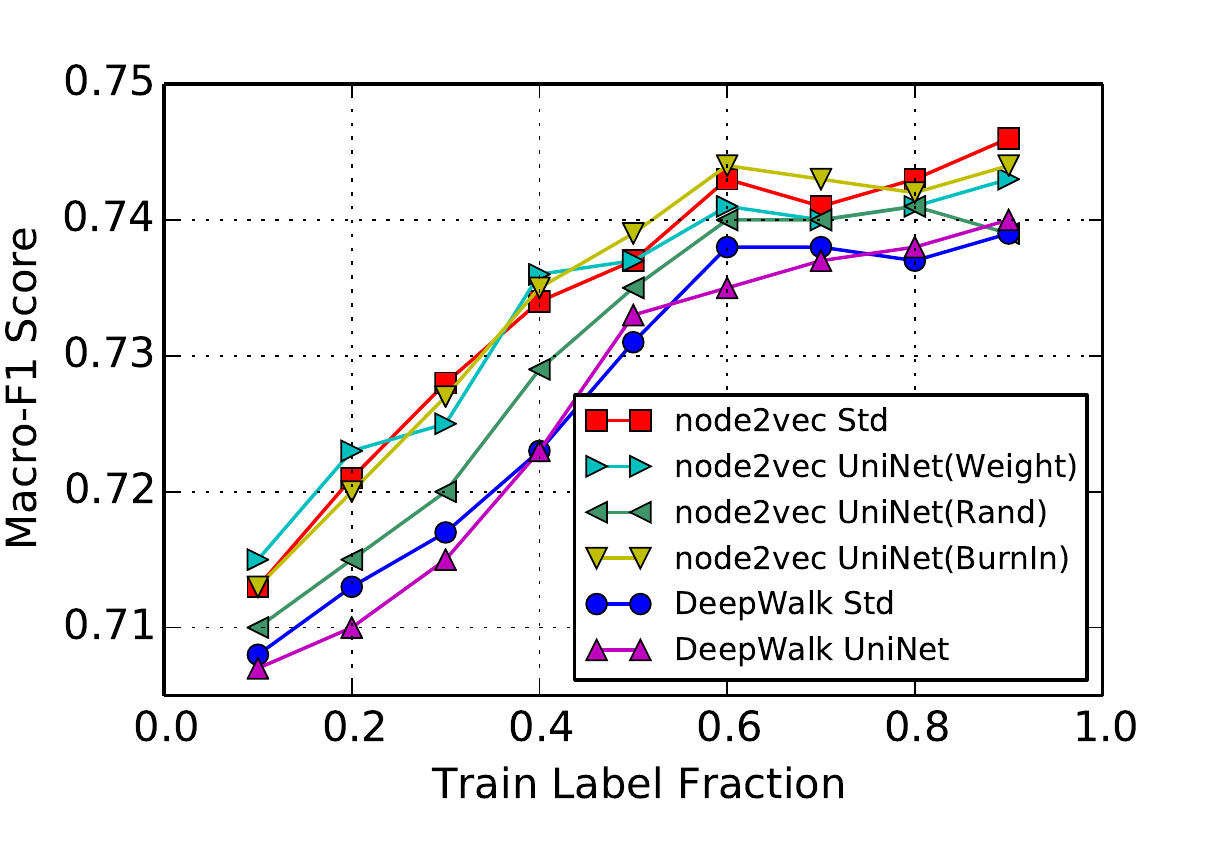}
    \label{subfig:reddit_macro}
     \end{minipage}%
    }%
    \subfigure[Aminer Macro-F1 Score]{
     \begin{minipage}[t]{0.24\linewidth}
    \centering
    \includegraphics[width=1.9in]{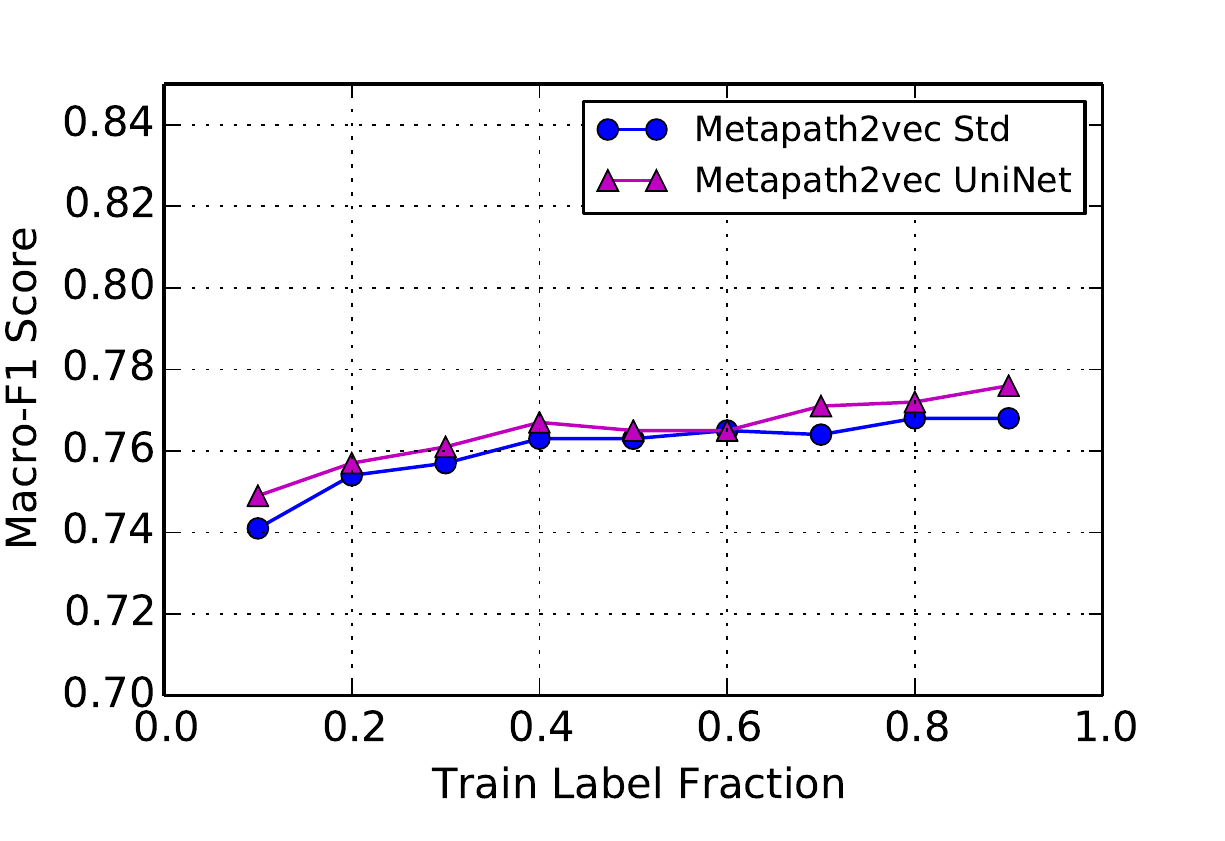}
    \label{subfig:aminer_macro}
     \end{minipage}%
    }%
    \centering
    \caption{The accuracy comparison among \sys and the original implementations on node classification task. \label{fig:acc}} 
    \end{figure*}
\textbf{Sampler Management.} \ 
Each M-H based edge sampler stores the last sample (i.e., $\textsc{Last}_{\mathbf{x}}$ in Algorithm~\ref{alg:sampler}) during the walking process.
The number of edge samplers in \sys equals the size of state space (e.g., $|\mathcal{V}|$ for deepwalk, $|\mathcal{E}|$ for node2vec). 
Such a great quantity of samplers requires proper management in order to meet the demand for fast queries by the state. 
However, the variety of state definitions introduces difficulties to the management of samplers. 
If a balanced tree is used to manage the samplers, each query entails $\mathcal{O}(\log(\#state))$ time complexity, which is
expensive for the random walk phase since each step of the walker (Line 9 in Algorithm~\ref{alg:walks}) requires searching over the samplers. 
Meanwhile, it is not trivial to hash the state for organizing the mapping of the samplers since the composition of the state is not fixed. 

To address this issue, we introduce a 2D data layout for the states as the index for fast querying. 
For the definition of the state of each model, we divide it into two components, \textit{position} and \textit{affixture}. 
The \textit{position} section represents the current residing node of the walker, 
while the \textit{affixture} section contains the extra information needed for the explicit definition of a state on the basis of \textit{position}. 
With regard to node2vec, \textit{affixture} indicates the index of the previously visited node among all the neighbors of the current node. 
As for metapath2vec, \textit{affixture} denotes the node type of current position in the metapath. 
For deepwalk, this section is empty.

On the basis of the decomposition above, we carefully reorganized the layout of the samplers. 
All samplers belonging to the state with the same \textit{position} are organized into a single bucket, in which each sampler is indexed by \textit{affixture}.
Fig.~\ref{fig:2d_data_layout} is an illustration of the layout of the samplers in node2vec.
When receiving a query based on a state, we first locate the bucket according to the \textit{position} in the state structure. 
After that, the \textit{affixture} attribute is utilized for determining the exact sampler from the sampler bucket. 
Such 2D data layout achieves aggregated storage of samplers, thus entailing $\mathcal{O}(1)$ time complexity for each query.

%% file: expr.tex
\section{Experimental Study}
\label{sec:expr}
\begin{table}[t]
    \centering
     \small
    \caption{Dataset Statistics.}
    \label{tab:data}
    \begin{tabular}{c|c|c|c|c}
    \hline
    \begin{tabular}[c]{@{}c@{}}Network\\ Dataset\end{tabular} & $|\mathcal{V}|$    & $|\mathcal{E}|$    & \begin{tabular}[c]{@{}c@{}}Mean\\ Degree\end{tabular} & \begin{tabular}[c]{@{}c@{}}\#Types \\of Nodes\end{tabular} \\ \hline
    BlogCatalog                                               & 10.3K  & 668K   & 64.9                                                  & 1                                                       \\ \hline

    Flickr                                                    & 80.5K  & 11.8M  & 146.6                                                 & 1                                                       \\ \hline
    Amazon                                                    & 335K  & 1.9M  & 5.67                                                 & 1                                                       \\ \hline
    {Reddit}                                                    & {231K}   & {11.6M}   & {50.21}                                                  & {1}
    \\ \hline
    YouTube                                                   & 1.1M   & 6.0M   & 5.3                                                   & 1                                                       \\ \hline
    LiveJournal                                               & 4.8M   & 86.2M  & 17.8                                                  & 1                                                       \\ \hline
    Twitter                                                   & 41.6M  & 2.9B   & 69.7                                                  & 1                                                       \\ \hline
    Web-UK                                                    & 105.9M & 6.6B   & 62.6                                                  & 1                                                       \\ \hline \hline
    ACM                                                       & 11.2K  & 34.8K  & 3.11                                                  & 3                                                       \\ \hline
    DBLP                                                      & 37.8K  & 341.6K & 9.04                                                  & 3                                                       \\ \hline
    
    DBIS                                                      & 134.1K & 530.6K & 3.96                                                  & 3                                                       \\ \hline
    AMiner                                                    & 4.9M   & 25.0M  & 5.10                                                  & 3                                                       \\ \hline
    \end{tabular}
  \end{table}
  \begin{table*}[t]
    \centering
    \color{black}
    \caption{{The time cost (seconds) of \sys and the original implementations for the five random walk based NRL models. 
    $T_i$, $T_w$, $T_l$ and $T_t$ are initialization cost, random walk generation cost, learning cost and total cost respectively. 
    \sys (Orig) represents the solution implementing the original sampling method of the models on \sys; \sys (M-H) is our proposed framework with the new M-H based edge sampler.
    `*' means out-of-memory error. `-' means the values are unavailable.}} 
    \label{tb:end_end_cmp}
    \scalebox{0.85}{
    \begin{tabular}{|c|ccc|c||ccc|c||ccc|c||c|c|}
    \hline
    NRL Model            & \multicolumn{4}{c||}{\begin{tabular}[c]{@{}c@{}}Open-sourced Version\end{tabular}} &  \multicolumn{4}{c||}{UniNet (Orig)}  & \multicolumn{4}{c||}{UniNet (M-H)}        & \multicolumn{2}{c|}{Speed-Up of $T_t$} \\ \hline\hline
    Deepwalk    & \multicolumn{1}{c|}{$T_i$}              & \multicolumn{1}{c|}{$T_w$} & \multicolumn{1}{c|}{$T_l$} & \multicolumn{1}{c||}{$T_t$}             &  \multicolumn{1}{c|}{$T_i$} & \multicolumn{1}{c|}{$T_w$}&  \multicolumn{1}{c|}{$T_l$} & $T_t$&  \multicolumn{1}{c|}{$T_i$} & \multicolumn{1}{c|}{$T_w$} &  \multicolumn{1}{c|}{$T_l$} &$T_t$  & $\frac{Orig(T_t)}{M\text{-}H(T_t)}$  &  $\frac{Open(T_t)}{M\text{-}H(T_t)}$     \\ \hline
    BlogCatalog &0.33&  16.47     &   8.34                         &    25.14           &   0.03     & 5.03 &     1.38            & 6.44& 0.02    &  0.11 &     1.38             & 1.51&\textbf{4.3X} & \textbf{16.6X}        \\ 
    Amazon &3.27&  687.23     &   254.52                         &    945.02           &   0.08     & 97.83 &     26.86            & 124.77& 0.53    &  11.87 &     24.19             & 36.59&\textbf{3.4X} & \textbf{26.1X}        \\ 
    Reddit      &5.38& 463.91     &  180.50                          &   649.79            &   0.09      & 360.35 &    21.05             & 381.49& 0.53    & 4.88 &    21.05             & 26.46  &\textbf{14.4X}  &   \textbf{24.6X}    \\
    Flickr      &5.01&  160.88    &     78.34                       &    244.26           &    0.05     & 190.24 &9.78& 200.07&  0.33          &   2.79   &9.78&12.9         & \textbf{15.5X}&\textbf{18.9X}\\ 
    YouTube     &3.35& 2010.4     &      1253.8                      &  3267.6             &  0.24       & 893.13 &   132.58               & 1025.95&  4.35   & 41.80 &   132.58               & 178.73   & \textbf{5.7X}&  \textbf{18.28X}  \\ 
    Twitter     & 130.29 & $>$4h &    -                       &   $>$4h               &   5.26    & $>$4h &    -            & $>$4h&  80.24   & 983.19   &    4983.2             &   6046.63     & -& - \\ 
    Web-UK      & * &  *     &   *                       &  *               &    8.54   & $>$4h &      -         & $>$4h&  100.21   &  1270.48  &      8637.9          &  10008.59     & -& -  \\ \hline\hline
    Node2vec    & \multicolumn{1}{c|}{$T_i$}              & \multicolumn{1}{c|}{$T_w$} & \multicolumn{1}{c|}{$T_l$} & \multicolumn{1}{c||}{$T_t$}             &  \multicolumn{1}{c|}{$T_i$} & \multicolumn{1}{c|}{$T_w$}&  \multicolumn{1}{c|}{$T_l$} & $T_t$&  \multicolumn{1}{c|}{$T_i$} & \multicolumn{1}{c|}{$T_w$} &  \multicolumn{1}{c|}{$T_l$}  & $T_t$  & $\frac{Orig(T_t)}{M\text{-}H(T_t)}$  &  $\frac{Open(T_t)}{M\text{-}H(T_t)}$ \\ \hline
    BlogCatalog& 1184.05  &    602.72  &   8.26                         &    1795.0          &    10.03         &  0.13    & 1.41&11.57  &  1.80 &   0.21                  & 1.41 &  1.80   &\textbf{6.4X}& \textbf{997.2X}   \\ 
    Amazon &1238.4&  707.23     &   163.48                         &    2109.1           &   18.48     & 3.01 &     23.84            & 45.33& 2.89    &  9.37 &     23.43             & 35.69&\textbf{1.27X} & \textbf{59.8X}        \\
    Reddit     & 8937.2 &   3304.2   &  201.24                          &    11442.6         &  240.53      & 3.97 &   27.48                     &271.98     &  2.98    &4.83  &   27.48                      & 35.29      &\textbf{7.7X} & \textbf{324.2X}  \\ 
    Flickr     & $>$4h &  -     &        -                   &              $>$4h &  230.16          &  2.03    & 9.69&     241.88   &     1.37                 &  1.80 & 9.69    & 12.86&\textbf{18.8X} &-\\ 
    YouTube    & 4580.9 &   $>$4h   &        -                   &          $>$4h   &    60.31        &   23.17   &86.45 & 169.93 &    150.09&    40.43                 & 86.45  &    150.09      & \textbf{1.1X}&-\\ 
    Twitter    & * &     *      &    *                   &    *                    &  *  &    *  & *& * &       1379.2             & 1705.7 & 4136.5& 7221.4      &- & - \\ 
    Web-UK     & *&    *        &    *                 &    *                      &  *  &  *  &  *    & *& 1506.1   &  2394.2                 &  8033.4    &  11933.7        & -&-\\ \hline\hline
    Metapath2vec    & \multicolumn{1}{c|}{$T_i$}              & \multicolumn{1}{c|}{$T_w$} & \multicolumn{1}{c|}{$T_l$} & \multicolumn{1}{c||}{$T_t$}             &  \multicolumn{1}{c|}{$T_i$} & \multicolumn{1}{c|}{$T_w$}&  \multicolumn{1}{c|}{$T_l$} & $T_t$ &  \multicolumn{1}{c|}{$T_i$} & \multicolumn{1}{c|}{$T_w$} &  \multicolumn{1}{c|}{$T_l$} &  $T_t$  & $\frac{Orig(T_t)}{M\text{-}H(T_t)}$  &  $\frac{Open(T_t)}{M\text{-}H(T_t)}$    \\ \hline
    ACM       & 0.38 &   7.87    &   3.99                        &   12.24     &    0.02          & 1.93 &  0.41         & 2.36 &    0.05      &   0.25    &  0.41          &    0.71     & \textbf{3.3X}&\textbf{17.2X} \\ 
    DBLP      & 1.02 &  32.90     &  7.26                         &  41.18      &    0.04          &16.19   &   0.56         & 16.79&    0.08       &  0.47    &   0.56          &  1.11         &\textbf{15.1X} &\textbf{37.1X} \\
    DBIS       & 5.49&  123.28     &     55.92                      &  184.69       &     0.07         &  13.04  &  11.13          & 24.24&  0.43   &   2.36    &  11.13          &   13.92      & \textbf{1.7X}&\textbf{13.3X}  \\ 
    Aminer     &19.2 &  4920.4     &    381.25                  & 5320.9    &    0.14     &   975.23         &     131.97       &1107.3       &   14.65   &  50.23    &     131.97        &  196.85       & \textbf{5.6X}& \textbf{27.0X}\\ \hline\hline
    Edge2vec    & \multicolumn{1}{c|}{$T_i$}              & \multicolumn{1}{c|}{$T_w$} & \multicolumn{1}{c|}{$T_l$} & \multicolumn{1}{c||}{$T_t$}             &  \multicolumn{1}{c|}{$T_i$} & \multicolumn{1}{c|}{$T_w$}&  \multicolumn{1}{c|}{$T_l$} & $T_t$&  \multicolumn{1}{c|}{$T_i$} & \multicolumn{1}{c|}{$T_w$} &  \multicolumn{1}{c|}{$T_l$}  &  $T_t$  & $\frac{Orig(T_t)}{M\text{-}H(T_t)}$  &  $\frac{Open(T_t)}{M\text{-}H(T_t)}$ \\ \hline
    ACM       & 1.38 &   260.87    &   3.99                        &  266.24      &  0.03            & 40.14&    0.30       & 40.47 &    0.18      &  0.34     &    0.30       & 0.82         &\textbf{49.4X} &\textbf{324.7X}\\ 
    DBLP      & 3.02 &  1845.2     &  7.26                         &  1855.5      &   0.04           & 64.19  &    0.62       & 64.85 &      0.53    & 1.07    &    0.62        &    2.22    &\textbf{29.2X} & \textbf{835.8X} \\
    DBIS      & 19.18 &   $>$4h    &   -                     &     $>$4h       &     0.09         &985.24 & 16.83 & 1002.2& 3.74                &  5.03          & 16.83  &   25.6    & \textbf{39.1X}&- \\ 
    Aminer    &312.47  &   $>$4h   &   -                     &    $>$4h        &      1.03        & $>$4h &    -    & $>$4h& 69.30  &   117.35        &    423.32      &       609.97   &- &-\\ \hline\hline
    Fairwalk    & \multicolumn{1}{c|}{$T_i$}              & \multicolumn{1}{c|}{$T_w$} & \multicolumn{1}{c|}{$T_l$} & \multicolumn{1}{c||}{$T_t$}             &  \multicolumn{1}{c|}{$T_i$} & \multicolumn{1}{c|}{$T_w$}&  \multicolumn{1}{c|}{$T_l$} & $T_t$&  \multicolumn{1}{c|}{$T_i$} & \multicolumn{1}{c|}{$T_w$} &  \multicolumn{1}{c|}{$T_l$}  & $T_t$    & $\frac{Orig(T_t)}{M\text{-}H(T_t)}$  &  $\frac{Open(T_t)}{M\text{-}H(T_t)}$    \\ \hline
    BlogCatalog& 4.20&  1985.5    &    8.97                       &   1998.7      &  0.07            &  37.07 &    1.83       & 38.97   & 0.23   &  0.29       &    1.83       &2.35        & \textbf{16.58X}& \textbf{850.9X} \\ 
    Amazon &5.07&  1984.3     &   372.96                         &    2362.3           &   0.09     & 93.81 &     23.97            & 117.87& 3.07    &  10.92 &     23.48             & 37.47&\textbf{3.15X} & \textbf{63.0X} \\
    Reddit     &56.01 &  $>$4h    &  -                         &  $>$4h      &      0.10       &   256.49   &     14.85       & 271.44&  7.35   & 9.30       &     14.85        & 31.50         & \textbf{8.6X}&-\\  \hline
    \end{tabular}
    }
\end{table*}
\subsection{Experimental Settings}
\sys is implemented in C++, and the default parallelism is set to 16.
The experiments are carried out on a server with a 24-core Xeon CPU and 96GB of memory.  
Each experiment is repeated five times and the averaged results are reported.

\subsubsection{Datasets}
We utilize eleven network datasets, of which seven are homogeneous networks. 
Table~\ref{tab:data}~summarizes the statistics for the datasets.

\subsubsection{Benchmarks}
We choose the random walk based NRL models listed in Table~\ref{tab:existing_models} as benchmarks.
For the deepwalk and node2vec, the hyper-parameters are set according to the paper of node2vec, while the other three models use the hyper-parameters suggested by the corresponding papers. In the random walk phase of all tests, each node is used as a starting point to generate 10 sequences of length 80. We use the open-sourced codes (see Table~\ref{tab:existing_models} ) released by the authors for comparison with \sys.

\subsection{Accuracy Evaluation of \sys}
\label{subsec:acc}
We evaluate the accuracy of NRL models on 
\sys by running the multi-label node classification task.
We use three models -- deepwalk, node2vec and metapath2vec as representatives. 
For deepwalk and node2vec, we run them on homogeneous networks, BlogCatalog, Flickr, and Reddit datasets. 
The parameters $(p, q)$ of node2vec are set to $(0.25, 4.0)$, $(0.25, 2.0)$, $(0.25, 0.25)$ over BlogCatalog, Flickr, and Reddit, respectively.
For metapath2vec, which operates on heterogeneous networks, we select Aminer for the accuracy evaluation.

 Fig.~\ref{fig:acc} visualizes the Micro-F1 and Macro-F1 of the three models. 
 We present the results of node2vec on \sys with three different initialization strategies, 
  because node2vec uses a biased transition probability distribution, 
  but only visualize the results of deepwalk and metapath2vec on \sys with the random initialization strategy, because the two models use uniform transition probability distribution, 
  the random and high-weight initialization strategies are the same according to Theorem~\ref{theo:high_prob}. 
  First, we can see that \sys achieves comparable accuracy compared with the original implementation of the three models. 
  The results demonstrate that our M-H based edge sampler is effective compared to the original sampling method, e.g., alias method or direct sampling method.
  Second, for node2vec, \sys with random initialization displays a slight drop in accuracy score in comparison with the original version,
  while \sys with high-weight initialization achieves better accuracy. Through profiling the characteristics of probability distributions across the network, we find that there are 97.1\%, 73.8\% and 87.3\% of nodes, whose probability distributions satisfy the condition~\ref{eq:init_cond} on BlogCatalog, Flickr and Reddit respectively. Therefore, according to Theorem~\ref{theo:high_prob}, \sys with high-weight initialization should achieve better accuracy than the one with random initialization. 

\begin{table*}[t]
    \centering
    \small
    \caption{The time cost (sec.) of random walk generation in node2vec model over billion-edge networks.}
    \label{tab:eff}
    \begin{tabular}{|c|c|c|c|c|c||c|c|c|c|c|}
    \hline
    Implementation        & \multicolumn{5}{c||}{Twitter}     & \multicolumn{5}{c|}{Web-UK}       \\ \hline
    $(p, q)$       & (1, 0.25)   & (0.25, 1)      & (1, 1)    & (1, 4)    & (4, 1)  & (1, 0.25) & (0.25, 1) & (1, 1)  & (1, 4)  & (4, 1)  \\ \hline
    Alias          & *           & *              & *         & *         & *       & *         & *         & *       & *       & *       \\ \hline
    Rejection      &   4228.02          &    11304.2            &  4092.19  & 10084.9   &  4157.18       & *         & *         & *       & *       & *       \\ \hline
    KnightKing      & 3601.43       & \textbf{1601.31}        & \textbf{1251.30}    & 9307.82    & 3310.29    & *   & *   & *  & * & * \\ \hline
    {Memory-Aware}      &  {4103.29}      &  {8059.83}       & {3982.45}    &  {8045.32}   & {4028.53}    & {6895.33}   & {12053.82}   & {5903.24}  & {11393.63} & {6023.64} \\ \hline
    UniNet(Rand)   &  \textbf{2535.48}    &  2468.39       &  2503.48  & \textbf{2493.29}   & \textbf{2539.40} & \textbf{2989.39}   & \textbf{2830.48}   & \textbf{3107.99} & \textbf{2846.49} & \textbf{3028.39} \\ \hline
    UniNet(Burn)   & 4363.32     & 4225.56        & 4376.47   & 4301.55   & 4378.56 & 6628.33   & 6273.48   & 6675.29 & 6518.90 & 6597.29 \\ \hline
    UniNet(Weight) & 3329.43     & 3702.18        & 2801.29   & 3245.10   & 3792.17 & 4829.30   & 5229.30   & 3184.28 & 3823.49 & 4592.19 \\ \hline
    \end{tabular}
\end{table*}

\subsection{Efficiency Evaluation of \sys}
In this subsection, we report the time cost of training the five NRL models with \sys over the eleven networks
in Table~\ref{tb:end_end_cmp}. 
{We not only compare our \sys to the open-sourced version of the models, but also compare it to the one implementing the original sampling methods of the models on \sys, denoted by \sys (Orig).
According to the open-sourced codes, except node2vec which uses the alias edge sampler, other four models use direct edge samplers. Our \sys uses M-H based edge sampler with high-weight initialization strategy by default.}

First, it is clear to see that \sys trains 10X-900X faster than the original versions of the models on various networks. Because of the unified framework, different random walk-based NRL models can benefit from the common optimization under the hood.
Second, the efficiency improvements of \sys on large networks or complex models are much more significant.
For example, \sys finishes training node2vec about 3.4 hours on Web-UK, while the original version fails because of the memory explosion problem;
for the metapath2vec model, \sys runs 27X faster than the original version over Aminer.
{Third, compared to \sys (Orig), on small and medium networks, the speed-up of the total cost is about 1.1X-49X, which is less significant than the aforementioned speed-up. This means the framework optimizations in \sys also help improve the efficiency of existing edge samplers. According to the results of node2vec, the M-H based edge sampler achieves better efficiency than the alias edge sampler because of the fast initialization. According to the results of other four models, the M-H based edge sampler outperforms the direct edge sampler due to the fast random walk generation (i.e., fast edge sampling). These observations are consistent with the characteristics of existing edge samplers described in Introduction.
Fourth, on large networks, \sys (Orig) still suffers from the expensive sampling cost or the memory explosion problem. In the next Section, we will compare our M-H based edge sampler with other state-of-the-art sampling methods for random walk generation over large networks.}

\subsection{{Efficiency over Billion-edge Networks}}
\label{subsec:large_networks}
{Here we compare the efficiency of different edge samplers for the five random walk based NRL models on two billion-edge networks -- Twitter and Web-UK. The edge samplers include alias edge sampler, rejection edge sampler, memory-aware edge sampler~\cite{ma_sigmod2020}, KnightKing~\cite{yang2019knightking} and M-H based edge sampler. 
KnightKing is the state-of-the-art solution for the fast distributed random walk generation. 
It uses rejection edge samplers, but adopts an outlier folding technique to efficiently handle the skewed probability distributions. For fairness, we configure KnightKing in the standalone mode. 
Memory-aware edge sampler is optimized for the second-order random walk generation, and it mainly generates an efficient sampler assignment within a certain memory budget. 
In the following experiments, we set the memory budget to be the same size as the memory consumption of \sys.}

\textbf{Node2vec}. For the node2vec, we select five configurations of node2vec hyper-parameters $(p, q)$. Table~\ref{tab:eff} shows the time span for random walk generation. From the table, we clearly see that \sys with M-H edge sampler achieves superior scalability, and it can successfully handle the two billion-edge networks on our server. {The memory-aware edge sampler is also able to handle the two datasets because of its memory-aware assignment strategy, but has low efficiency.} Alias sampler fails due to the out-of-memory issue on both networks. The rejection edge sampler and KnightKing fail on Web-UK due to the same memory problem. 
The concrete reason is that: the rejection edge sampler and KnightKing use the rejection-acceptance paradigm, which requires a proposal distribution (e.g., the probability distribution defined by the static edge weight). Although proposal distribution is simple, to ensure the efficiency, they are still sampled by the alias method. 
As the size of the network grows, the memory storage related to the proposal distribution becomes the bottleneck. 
Instead, our M-H edge sampler uses the uniform probability distribution as the conditional probability mass function, which can be efficiently sampled without using the alias method, so it reduces the consumption of memory significantly. {On Twitter network, with some hyper-parameter configurations, KnightKing can be better than M-H edge sampler in \sys, because of the outline folding technique and its system optimizations.}

\textbf{Other four models}. 
{Deepwalk and metapath2vec use first-order random walk models, so we do not evaluate them with the memory-aware edge sampler. Fairwalk and edge2vec are variants of node2vec, and we set the corresponding $(p,q)$ to be $(1, 1)$ and $(0.25, 0.25)$ respectively following their original papers. Since KnightKing does not provide the implementations of fairwalk and edge2vec, 
we implemented them by ourselves following KnightKing's interfaces. 
For the models using heterogeneous networks, we adopt the method in work~\cite{yang2019knightking} to randomly generate type information for the networks.}

\begin{figure}[t]
    \centering
    \subfigure[Twitter]{
     \begin{minipage}[t]{0.45\linewidth}
    \centering
    \includegraphics[width=1.13\textwidth]{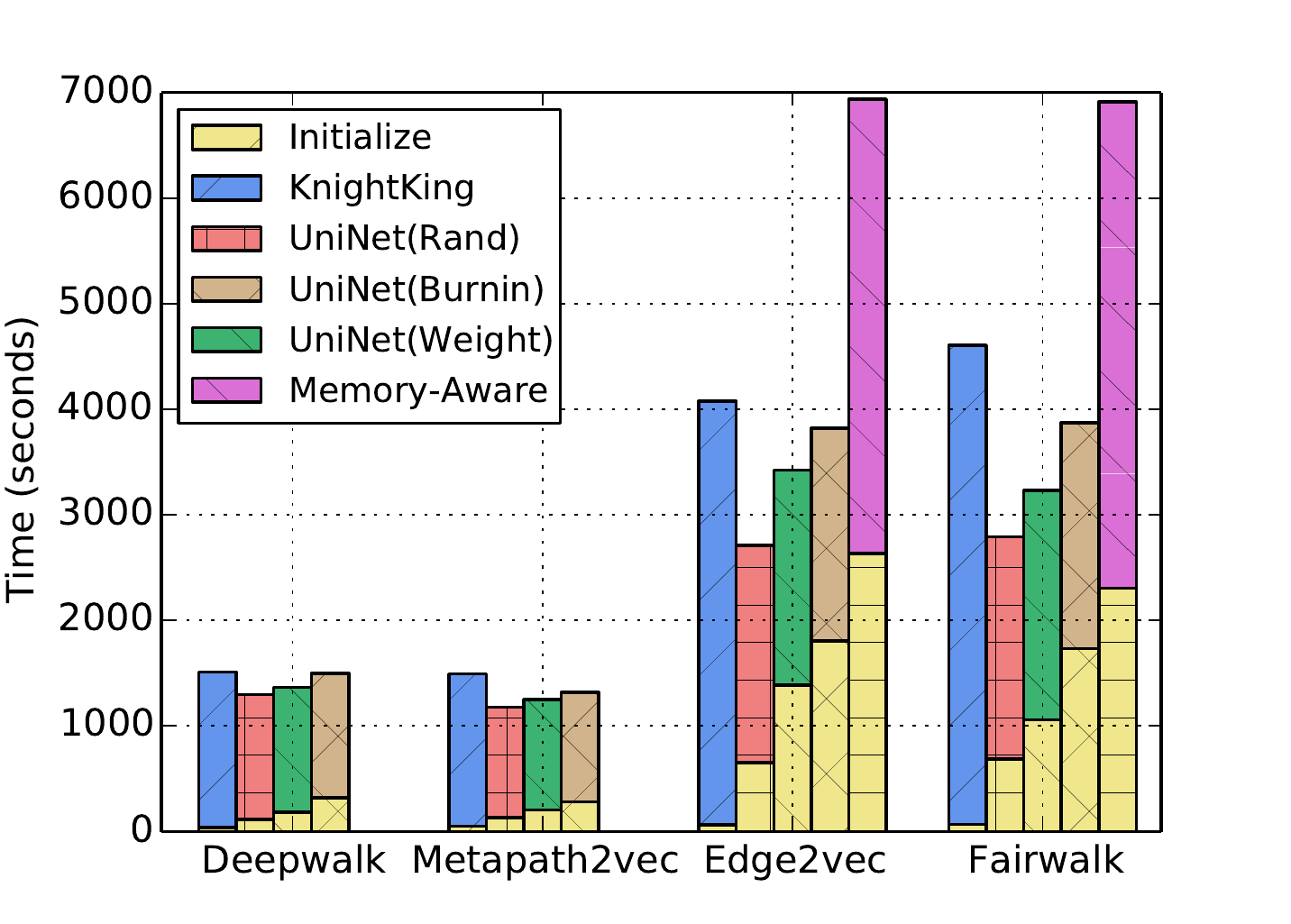}
    \end{minipage}%
    }
    \subfigure[Web-UK]{
    \begin{minipage}[t]{0.55\linewidth}
    \centering
    \includegraphics[width=.86\textwidth]{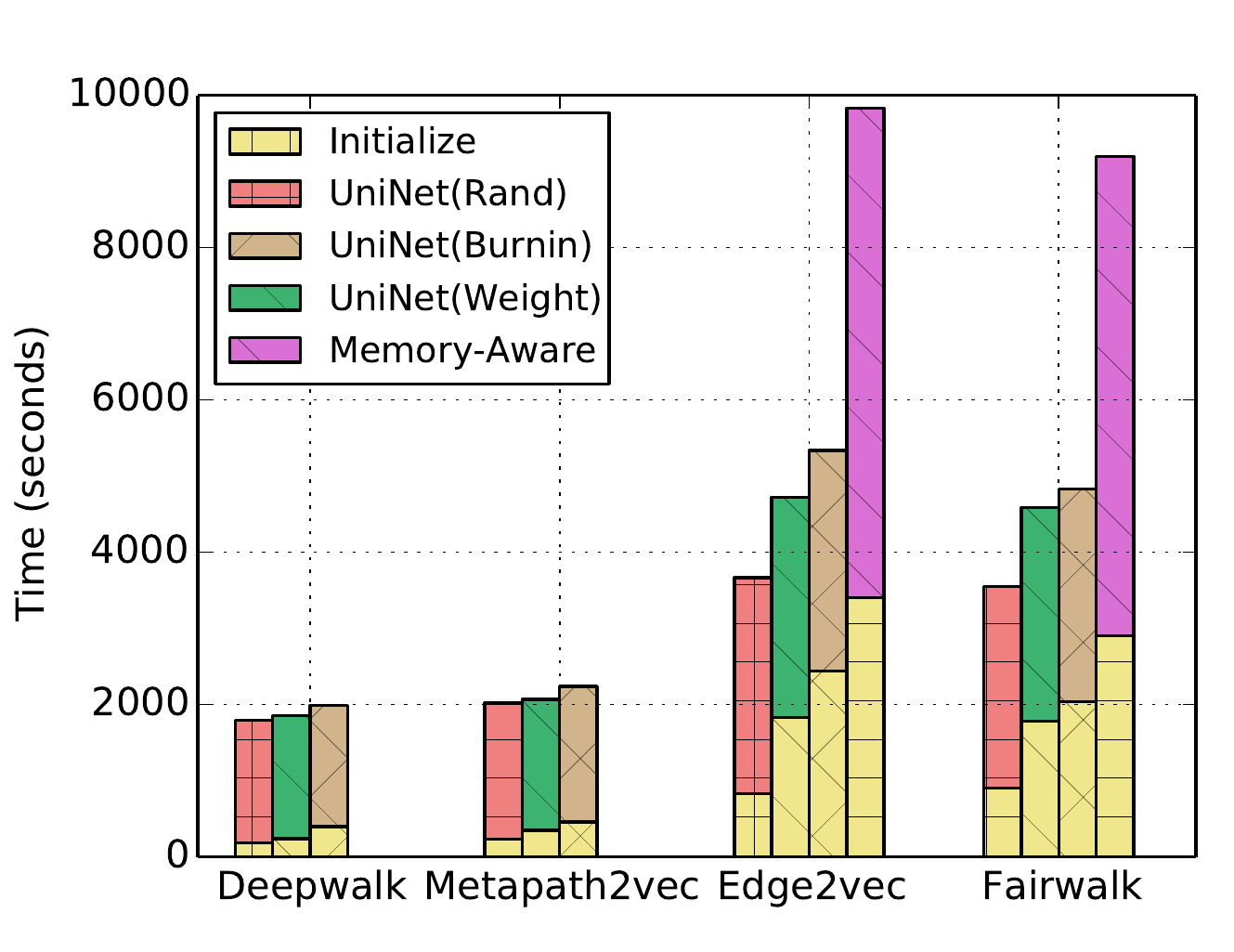}
    \end{minipage}%
    }%
    \centering
    \caption{{The time cost (sec.) of random walk generation in deepwalk, metapath2vec, edge2vec and fairwalk.} \label{fig:four_models_large_graph}} 
\end{figure}

\begin{figure*}[t]
    \centering
    \subfigure[\scriptsize {Varying $p$ in node2vec on LiveJournal}]{
    \label{mi-b1}
    \begin{minipage}[t]{0.24\linewidth}
    \centering
    \includegraphics[width=1.9in]{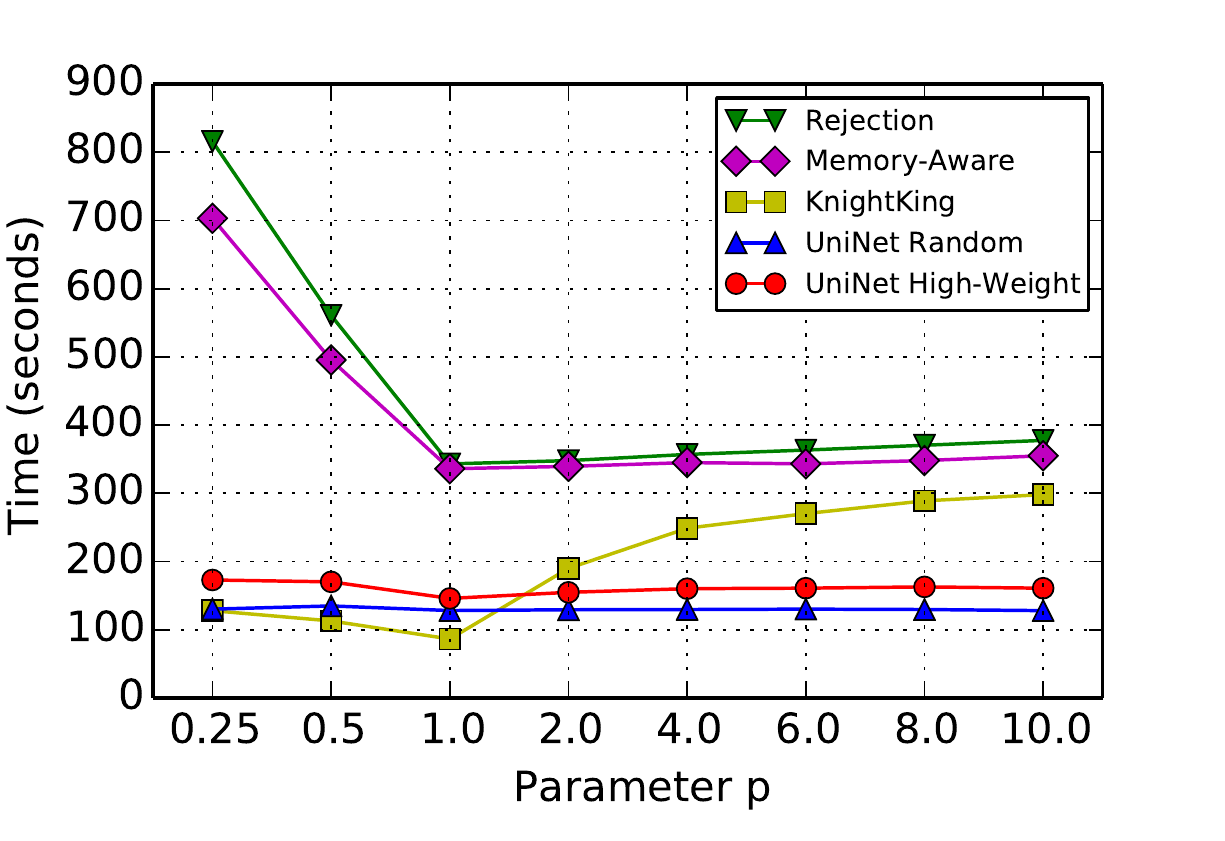}
    \end{minipage}%
    }%
    \subfigure[\scriptsize {Varying $q$ in node2vec on LiveJournal} ]{
    \label{mi-b2}
    \begin{minipage}[t]{0.24\linewidth}
    \centering
    \includegraphics[width=1.9in]{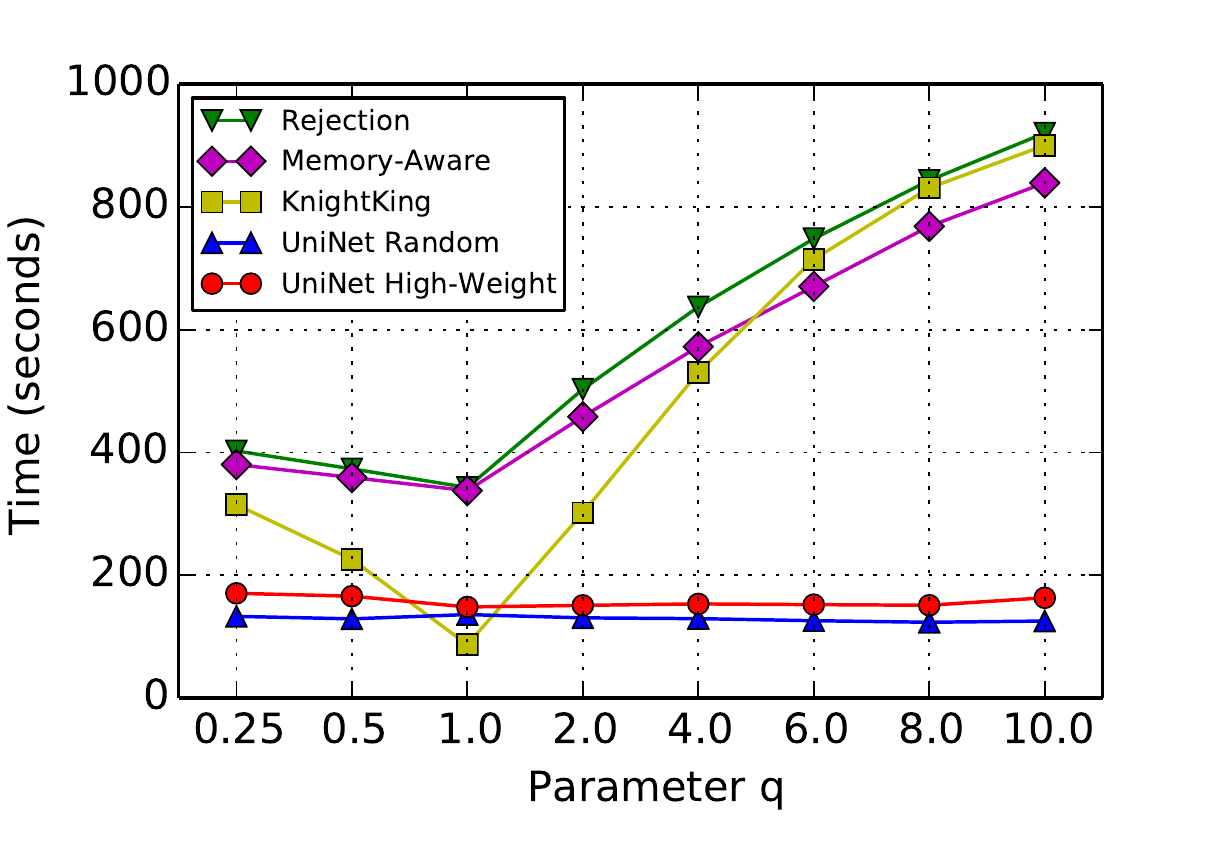}
    \end{minipage}%
    }%
    \subfigure[\scriptsize {Varying $p$ in edge2vec on Aminer}]{
    \label{mi-b3}
    \begin{minipage}[t]{0.24\linewidth}
    \centering
    \includegraphics[width=1.9in]{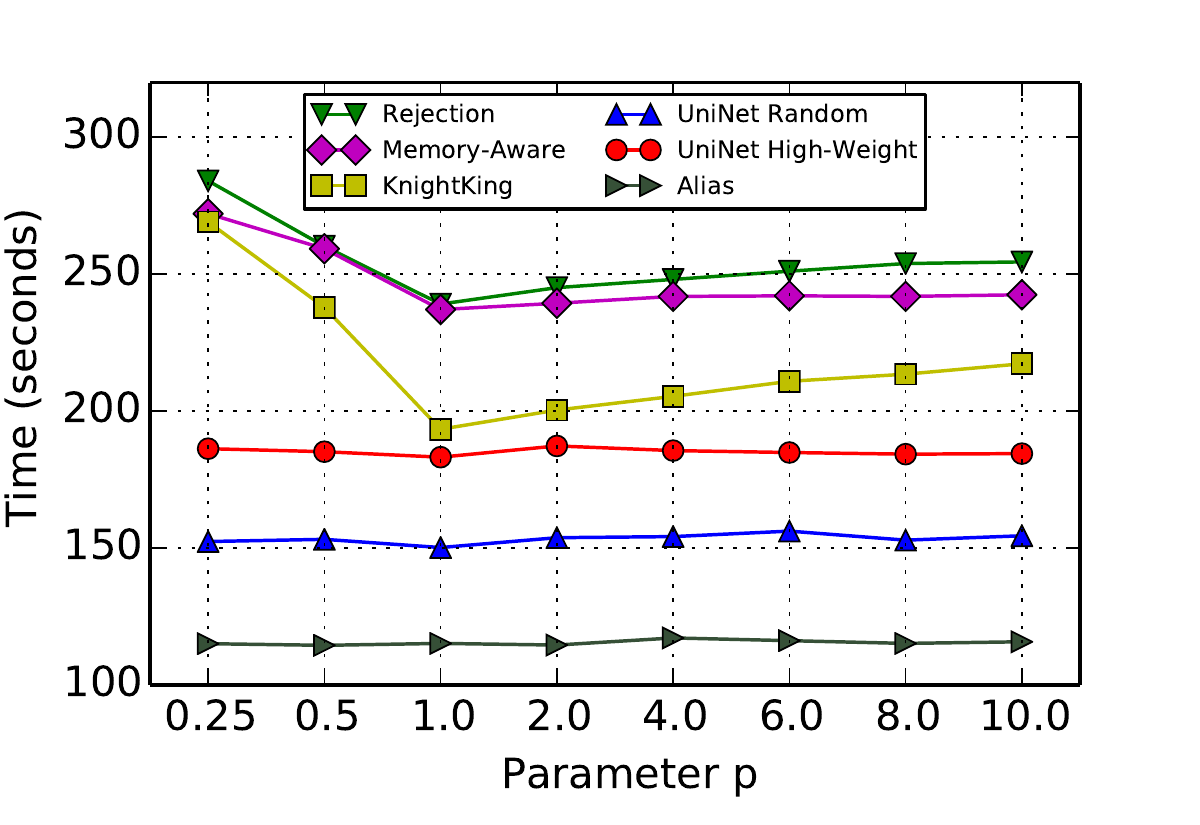}
    \end{minipage}%
    }%
    \subfigure[\scriptsize {Varying $q$ in edge2vec on Aminer}]{
    \label{mi-b4}
    \begin{minipage}[t]{0.24\linewidth}
    \centering
    \includegraphics[width=1.9in]{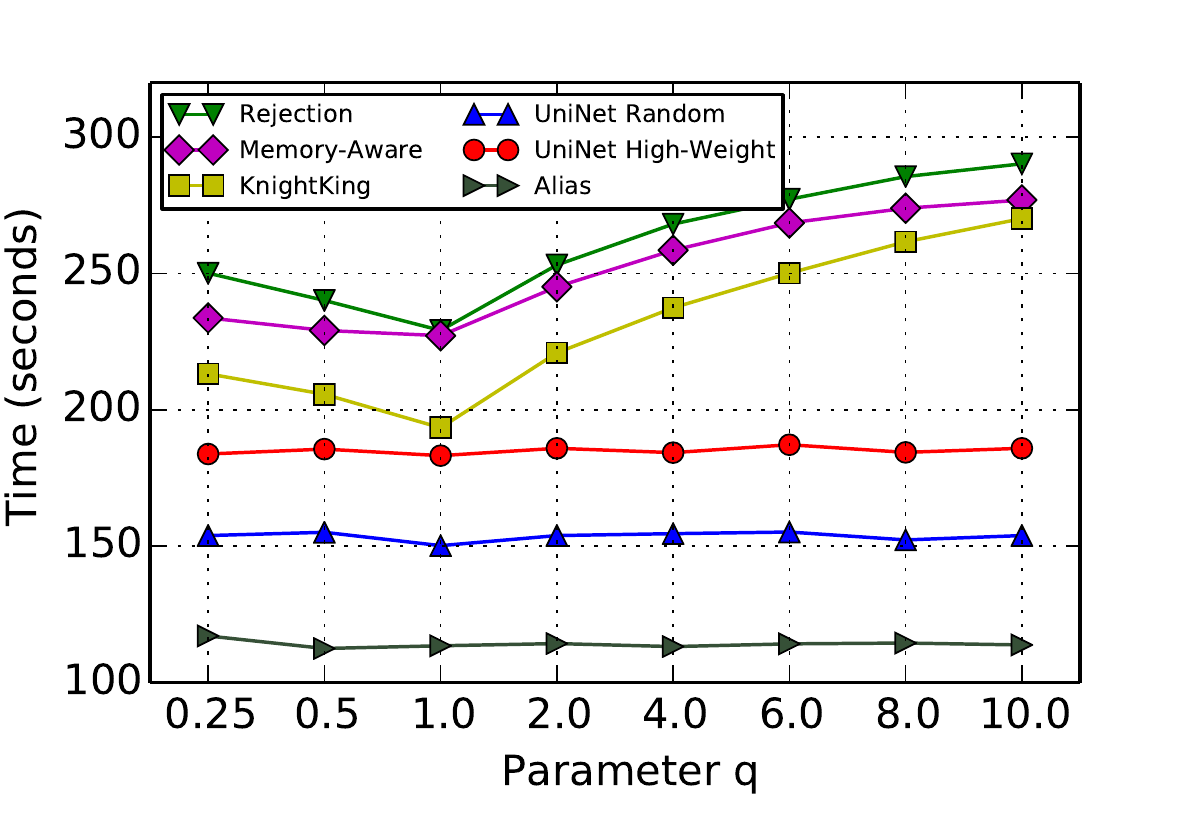}
    \end{minipage}%
    }%
    
    \subfigure[\scriptsize {Varying $p$ in node2vec on Youtube}]{
    \label{mi-b5}
    \begin{minipage}[t]{0.25\linewidth}
    \centering
    \includegraphics[width=1.78in]{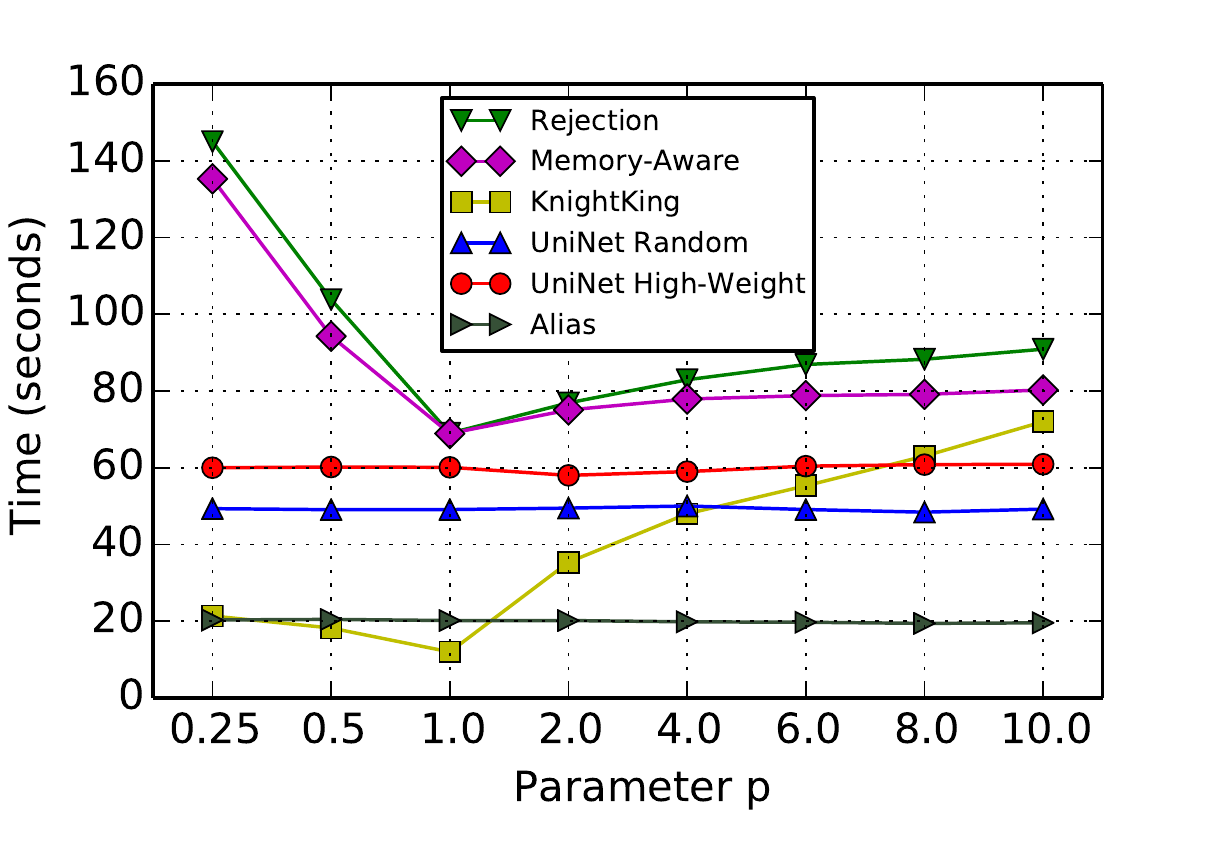}
    \end{minipage}%
    }%
    \subfigure[\scriptsize {Varying $q$ in node2vec on Youtube} ]{
    \label{mi-b6}
    \begin{minipage}[t]{0.25\linewidth}
    \centering
    \includegraphics[width=1.78in]{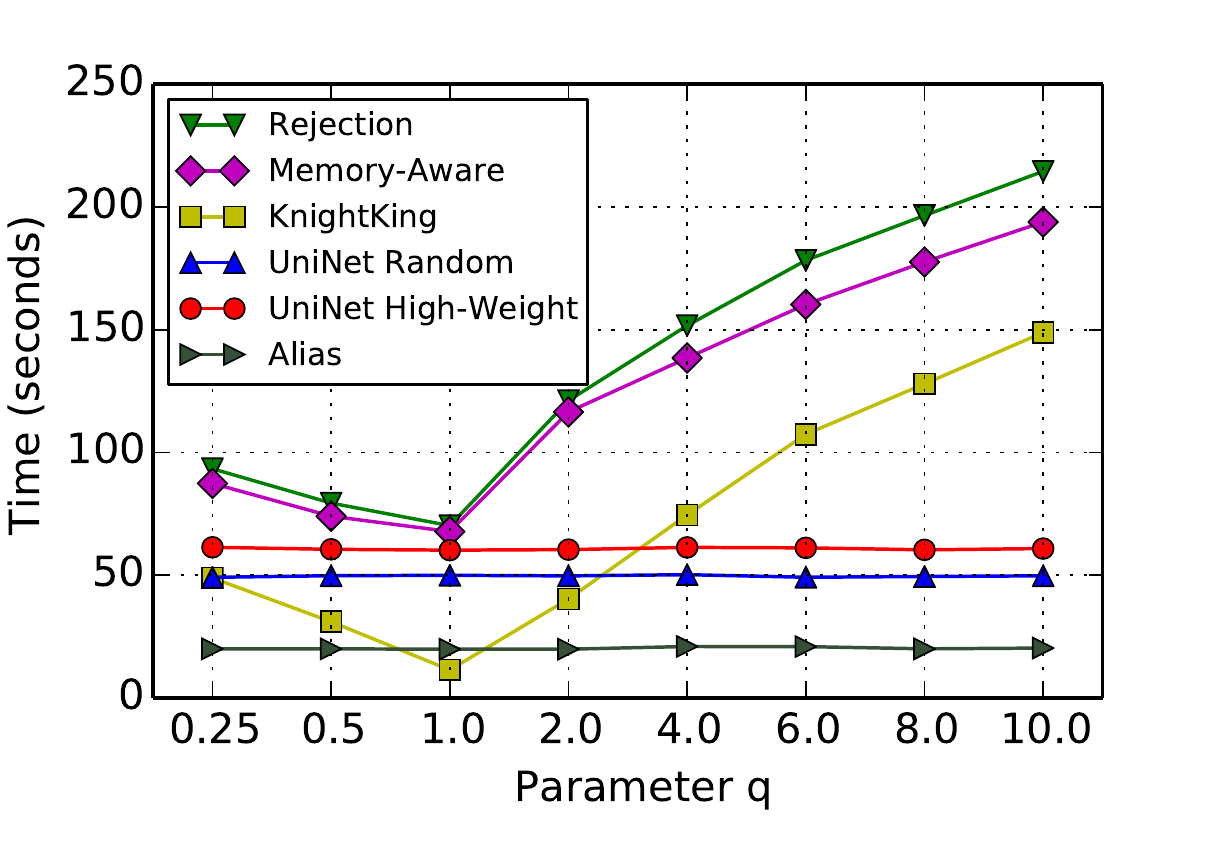}
    \end{minipage}%
    }%
    \subfigure[\scriptsize {Varying $p$ in fairwalk on YouTube}]{
    \label{mi-b7}
    \begin{minipage}[t]{0.24\linewidth}
    \centering
    \includegraphics[width=1.9in]{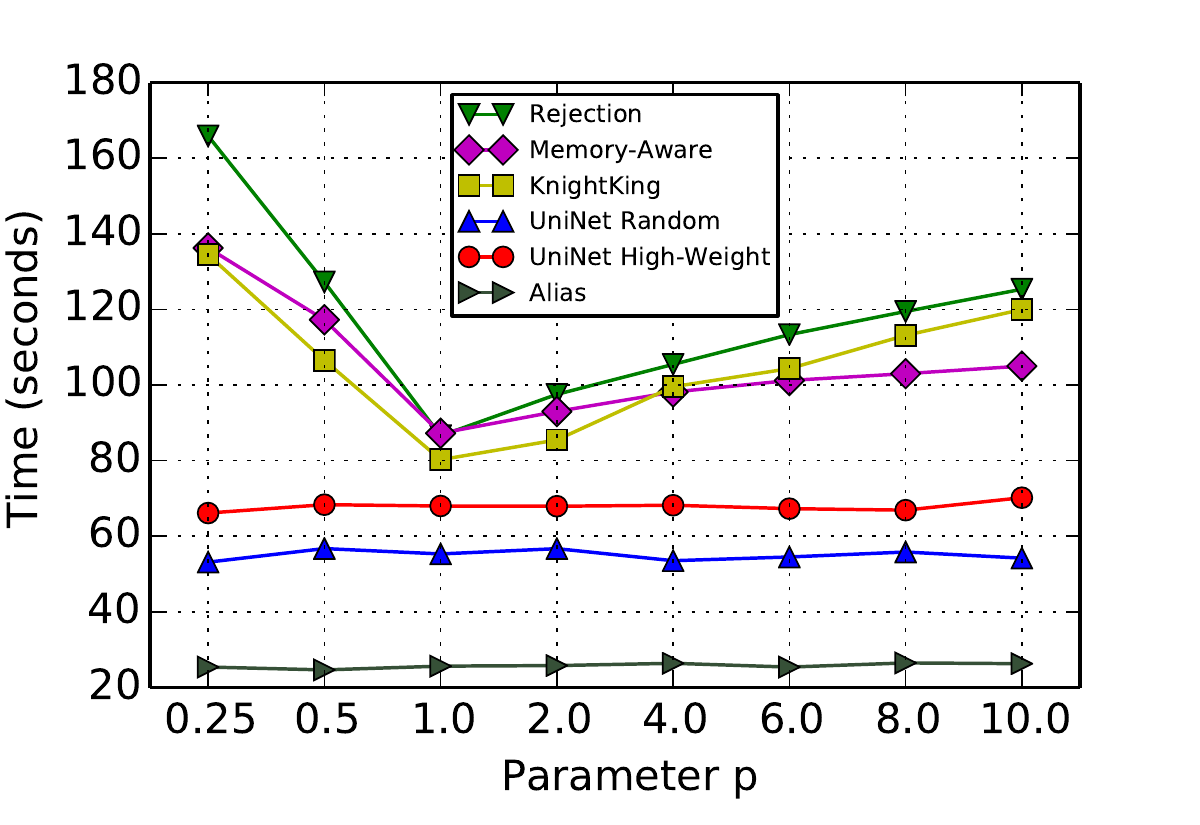}
    \end{minipage}%
    }%
    \subfigure[\scriptsize {Varying $q$ in fairwalk on YouTube}]{
    \label{mi-b8}
    \begin{minipage}[t]{0.24\linewidth}
    \centering
    \includegraphics[width=1.9in]{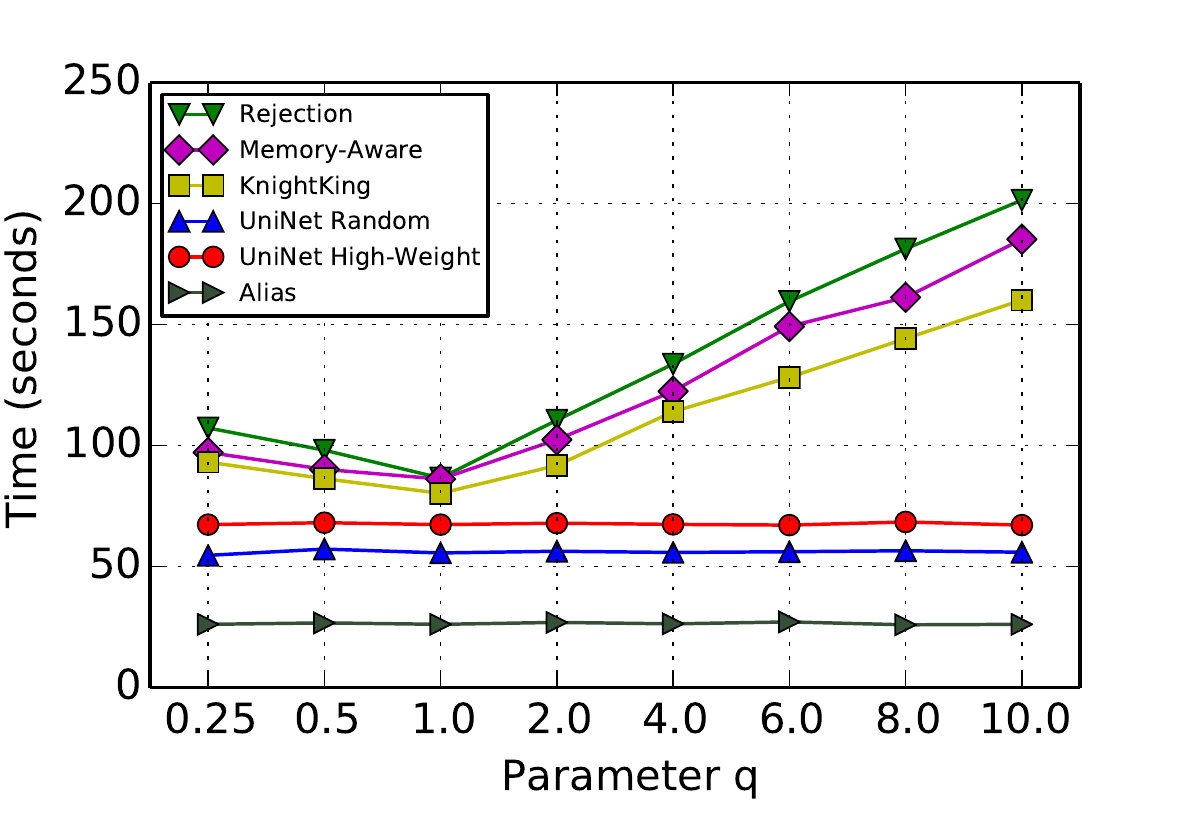}
    \end{minipage}%
    }%
    \centering
    \caption{{Parameter sensitivity of different edge samplers.}}
    \label{fig:parameter_sensitivity}
    \end{figure*}

{Fig. \ref{fig:four_models_large_graph} visualizes the total time costs of the four models with different edge samplers. In the figure, we decompose the total time cost into initialization cost and walking cost. Each yellow bar represents the initialization cost. The results of KnightKing on Web-UK is not shown because of the out-of-memory problem.
From the figure, we can see that in the aspect of the first-order models, i.e., deepwalk and metapath2vec, the total time cost of \sys is close to KnightKing. This is because KnightKing employs alias sampling for the first-order random walk, which entails $\mathcal{O}(1)$ time complexity. Meanwhile KnightKing suffers from the high memory problem causing by the alias sampling, thus failing to handle much larger networks like Web-UK on a single machine. 
For fairwalk and edge2vec, \sys performs the best. Unlike the one in node2vec, the probability distributions in fairwalk and edge2vec are influenced by both the hyper-parameters and the heterogeneous information, and it is hard to figure out a deterministic way to describe the outlier cases in the probability distributions. As a result, KnightKing cannot benefit from the outlier folding technique. The efficiency of memory-aware edge sampler is worse than the ones of KnightKing and \sys, because it only schedules classical samplers within a memory budget and does not optimized for the skewed probability distributions. Despite the low efficiency, the memory-aware edge sampler is able to handle the Web-UK dataset, while KnightKing fails due to memory shortage.}

To summarize, compared to KnightKing and memory-aware sampler, our M-H based edge sampler can achieve 9.6\%-73.2\% efficiency improvement in most of the settings and handle much larger networks.

\textit{The comparison of different initialization costs in M-H edge sampler}. 
{The initialization procedure is critical for M-H edge sampler in efficiency. 
From Fig. \ref{fig:four_models_large_graph}, we can see that in fairwalk and edge2vec, the initialization cost of burn-in strategy takes up 42\%-47\% of the total cost. 
By applying the random and high-weight initialization strategies, 
the ratio is reduced to 24\%-40\%. This is consistent with our analysis in Section~\ref{subsec:init}, and the burn-in strategy requires a large number of iterations of sampling action for the initialization. In our experiments, the number is only set to 100 after parameter tuning, i.e., a smaller number will lead to the model accuracy loss of \sys. Moreover, in conjunction with the results in Section~\ref{subsec:acc}, high-weight strategy entails relatively low initialization overhead compared with burn-in, meanwhile maintaining the model accuracy.}

\subsection{{Sensitivity} of the M-H based Edge Sampler}
In the Introduction, we claim that our edge sampler is stable while rejection edge sampler is parameter sensitive. {To verify this, we conduct sensitivity experiments with node2vec, fairwalk and edge2vec since their transition probability distributions can be easily tuned by varying two hyper-parameters $p$ and $q$.} Specifically, for the two hyper-parameters, we fix one hyper-parameter to be 1, and then vary the other in [0.25, 0.5, 1.0, 2.0, 4.0, 6.0, 8.0, 10] to produce different transition probability distributions. Fig.~\ref{fig:parameter_sensitivity} illustrates the experimental results. We execute node2vec on YouTube and LiveJournal, fairwalk on YouTube, and edge2vec on Aminer.
The results of alias edge sampler are not visualized on LiveJournal because of the out-of-memory problem.

\textbf{1)} It is clear to see that the efficiency of M-H edge sampler in \sys and alias edge sampler remains stable,
while the time costs of {memory-aware edge sampler}, rejection edge sampler and KnightKing vary a lot. 
In Fig~\ref{mi-b2}, when changing $q$ in node2vec on LiveJournal from 1.0 to 10,
\sys with High-Weight entails around {160s} to finish the workload, 
but the time cost of KnightKing increases from {87s to 900s.} The results of edge2vec and fairwalk are similar.
It is because, as analyzed in Section~\ref{subsec:sample}, the time cost of M-H based edge sampler in \sys is only affected by the computation of dynamic weight edge, and is not related to the value of hyper-parameters. 
In contrast, the time complexity of rejection-based edge samplers is related to the shape of probability distribution, which is heavily influenced by the hyper-parameters.

\textbf{2)} The efficiency of KnightKing with varying $q$ in node2vec is more unstable than the one with varying $p$. This is because, according to Eq.~\ref{equ:node2vec}, parameter $q$ affects multiple dynamic edge weights, which leads to multiple outliers, while parameter $p$ affects exact one dynamic edge weight. This results indicate that the outlier folding technique is more effective when the number of outliers is small. 
{In Figure~\ref{mi-b3} and~\ref{mi-b7}, unlike node2vec, when $p$ is smaller than 1.0, the KnightKing on edge2vec and fairwalk is inefficient as well. As discussed in Section~\ref{subsec:large_networks}, the properties of the transition probability distributions in edge2vec and fairwalk are also influenced by the heterogeneous information which makes them non-deterministic, so users cannot predefine the outlier cases (easily), making the outlier folding technique ineffective.}

\textbf{3)} KnightKing outperforms the rejection edge sampler, {and is better than \sys with some hyper-parameters on node2vec}, because of the sophisticated systematic optimizations and the optimizations in terms of rejection sampling algorithm including pre-acceptance and the outlier folding technique, while other samplers are implemented in \sys, which is a simple prototype for proof-of-concept. 
{In addition, on edge2vec and fairwalk, due to the non-deterministic outlier cases in the transition probability distributions, KnightKing performs worse than \sys.}

%% file: relate.tex
\section{Related Work}
\subsection{Network Representation Learning}

Besides the random walk based NRL models studied in our paper, another emerging category is deep graph neural network~\cite{kipf2016semi}, which uses deep neural network to learn network representation. 
The model development in both categories is usually conducted over small datasets, most of existing models have poor scalability to the large scale datasets. 
Recently, some researches begin to optimize the efficiency of random walk NRL models. 
Zhou et al.~\cite{zhou2018efficient} introduced a distributed solution for node2vec with a pregel-like graph computation framework.
It computes the transition probability during the random walk to reduce memory space, and decrease the computation overhead with sophisticated optimization techniques. This solution requires a cluster to handle billion-edge graphs. 
Similarly, PyTorch-BigGraph~\cite{mlsys2019_71} proposed by Facebook uses graph partitioning technique and parameter server architecture to learn embeddings from large multi-relation networks, and introduces several modifications to the standard models.
Our \sys is a single node solution, and is optimized for general random walk based NRL models.

\subsection{{Metropolis-Hastings Sampling Techniques}}
{Metropolis-Hastings (M-H) algorithm is first proposed by Nicholas Metropolis for the case of symmetrical proposal distributions~\cite{Metropolis1953}.
A basic problem of the M-H algorithm is convergence. Roberts et al.~\cite{roberts1994simple} introduced simple conditions for the convergence of M-H algorithms. The works~\cite{mengersen1996rates,HOLDEN1998371,chauveau2006selection} further theoretically analyzed the conditions for the geometric convergence. 
However, up to now, it is still an open question to show that how many iterations the M-H algorithms require to be converged in general, thus bringing in the ``burn-in'' period for initialization. In this paper, our M-H based edge sampler uses two new efficient heurestic intialization strategies instead of the expensive ``burn-in'' period. We establish theorems which guide users to properly choose initialization strategy with regard to the target probability distributions.}


{Moreover, in the context of networks, Metropolis-Hastings random walk~\cite{lovasz1993random,7113345} is often used to (uniformally) sample nodes. It describes a simple way to modify the random walk, so that it converges to a prescribed node probability distribution. The M-H based edge sampler in \sys, to the best of our knowledge, is the first method applying the M-H sampling technique to converge to the transition probability distributions.}

\subsection{{Fast Random Walk in Networks}}

{Performing random walk generation efficiently in networks is a fundamental task.
Atish et al.~\cite{dist_rw_2009, dist_rw_2015} proposed several fast distributed random walk algorithms, 
which minimize the number of rounds and messages required in distributed settings.
Zhou et al.~\cite{zhou2018efficient} introduced a distributed second-order random walk algorithm on a Pregel-like graph computation framework~\cite{10.1145/1807167.1807184,6823116} to speed up the node2vec model.} Yang et al.~\cite{yang2019knightking} proposed a new distributed random walk engine on the basis of optimized rejection sampling method, which exhibits promising performance while ensuring low memory consumption. 
Besides the distributed solutions, there also exist serveral efficient single-node solutions. 
DrukarMob~\cite{kyrola2013drunkardmob} is a specialized system for fast random walk running out-of-core on a single machine with multi-core processors. DurkarMob only focuses on the standard walk and optimizes network organization to achieve performance improvement. 
{The memory-aware random walk framework~\cite{ma_sigmod2020} dynamically assigns different samping methods for the nodes to miminize the random walk  time cost within the memory budget, but it is only optimized for the second-order random walk models.
\sys in this paper improves the efficiency of random walk generation by applying the new M-H based edge sampler, which further helps unify various random walk based NRL models and scale them to large graphs.}

%% file: conc.tex
\section{Conclusion}
\label{sec:conc}
Random walk based NRL model is an effective approach to learn structural information of networks nowadays.
In this work, we were concentrated on unifying existing random walk based NRL models into an optimized framework \sys, which supports the models to efficiently learn from large networks, especially billion-edge networks. We adopted Metropolis-Hastings sampling for edge sampling, dramatically enhancing the efficiency of random walk generation phase in the NRL models. In addition, we theoretically discussed the impact of different initialization strategies of our edge sampler for the accuracy of the model, and carefully introduced two initialization strategies to balance the efficiency and accuracy. Finally, to unify the NRL models, we introduced a random walk model abstraction, which allows users to define different random walk models on demand.
The empirical studies clearly demonstrated that \sys trains random walk based NRL models significantly faster than existing solutions.

%% file: appendix.tex
\begin{appendices}
\section{The Proof of Theorem~\ref{theo:high_prob}}
\label{app:proof}
\small
\begin{proof}
    \textbf{High-weight initialization}. Assume there are $t$ different elements whose probability is $\pi_{max}=\max\limits_{x\in\Omega}(\pi(x))$, and each element can be sampled uniformly, then 
    \begin{equation}
        \pi^0(x) = \left\{
            \begin{array}{ll}
                \frac{1}{t} & x=\mathop{\arg\max}\limits_{x\in\Omega}(\pi(x)), \\
                0 & otherwise .
            \end{array}
        \right.
    \end{equation}
    From Eq.~\ref{eq:kappa} and the definition of supreme norm, we have  
    \begin{equation}
        \kappa_h = \max \{ \frac{1}{t\pi_{\mathrm{max}}} - 1,1 \}.
    \end{equation}

    \textbf{Random Initialization}.
    Accordingly, if $\pi^0$ is uniform, i.e., $\pi^0(x)=\frac{1}{n}$, we have 
    \begin{equation}
        \kappa_r = \max \{ 1 - \frac{1}{n\pi_{\mathrm{max}}},\frac{1}{n\pi_{\mathrm{min}}} - 1\}.
    \end{equation}

    To figure out the conditions for $\kappa_h < \kappa_r$, we discuss by cases.
    
    \textbf{Case 1}: $\frac{1}{t\pi_\mathrm{max}} - 1 > 1$, which leads to $\pi_\mathrm{max} < \frac{1}{2t}$. In this case,  
    \begin{equation}
        \kappa_h = \frac{1}{\pi_\mathrm{max}} - 1.
    \end{equation}

    Case 1.1: $1 - \frac{1}{n\pi_{\mathrm{max}}} > \frac{1}{n\pi_{\mathrm{min}}} - 1$. Let $\kappa_h < \kappa_r$, we can get  
    \begin{equation}
        \pi_\mathrm{max} > \frac{n + t}{2nt} > \frac{1}{2t},
    \end{equation}
    which contradicts with the precondition that $\pi_\mathrm{max} < \frac{1}{2t}$.

    Case 1.2: $1 - \frac{1}{n\pi_{\mathrm{max}}} < \frac{1}{n\pi_{\mathrm{min}}} - 1$. 
    Let $\kappa_h < \kappa_r$, we can get 
    \begin{equation}
        \label{eq:res_1}
        \frac{\pi_\mathrm{max}}{\pi_\mathrm{min}} > \frac{n}{t}.
    \end{equation}
    From Eq. \ref{eq:res_1} and $\pi_\mathrm{max} < \frac{1}{2t}$ it can be deduced that the prerequisite condition for case 1.2 is guaranteed to be satisfied.
    
    \textbf{Case 2}: $\frac{1}{t\pi_\mathrm{max}} - 1 < 1$, which leads to $\pi_\mathrm{max} > \frac{1}{2t}$. In this case, $k_h = 1$.
    
    Case 2.1: $1 - \frac{1}{n\pi_{\mathrm{max}}} < \frac{1}{n\pi_{\mathrm{min}}} - 1$. 
    Let $\kappa_h < \kappa_r$, we can get 
    \begin{equation}
        \pi_{\mathrm{min}} < \frac{1}{2n},
    \end{equation}
    which directly fulfills the prerequisite condition for case 2.1.
  
    Case 2.2: $1 - \frac{1}{n\pi_{\mathrm{max}}} > \frac{1}{n\pi_{\mathrm{min}}} - 1$. The condition $\kappa_h < \kappa_r$
    The conclusion can be derived from condition $\kappa_h > \kappa_r$ is impossible since $\kappa_h = 1 >1 - \frac{1}{n\pi_{\mathrm{max}}} = \kappa_r$.
       
    In summary, if $\kappa_h < \kappa_r$, then
    
    \begin{equation}
        \left\{
            \begin{array}{l}
                \pi_{\mathrm{max}} < \frac{1}{2t} \\
                \frac{\pi_{\mathrm{max}}}{\pi_{\mathrm{min}}} > \frac{n}{t} 
            \end{array}
        \right.
        or \
        \left\{
            \begin{array}{l}
                \pi_{\mathrm{max}} \ge \frac{1}{2t} \\
                \pi_{\mathrm{min}} < \frac{1}{2n} 
            \end{array}
        \right. 
    \end{equation}
\end{proof}
\end{appendices}

%% file: main.bbl
\begin{thebibliography}{10}

\bibitem{bengio2013representation}
Yoshua Bengio, Aaron Courville, and Pascal Vincent.
\newblock Representation learning: A review and new perspectives.
\newblock {\em TPAMI}, 35(8):1798--1828, 2013.

\bibitem{chauveau2006selection}
Didier Chauveau and Pierre Vandekerkhove.
\newblock Selection of a mcmc simulation strategy via an entropy convergence
  criterion.
\newblock {\em arXiv preprint}, 2006.

\bibitem{chib1995understanding}
Siddhartha Chib and Edward Greenberg.
\newblock Understanding the metropolis-hastings algorithm.
\newblock {\em The american statistician}, 49(4):327--335, 1995.

\bibitem{dist_rw_2015}
Atish Das~Sarma, Anisur~Rahaman Molla, and Gopal Pandurangan.
\newblock Efficient random walk sampling in distributed networks.
\newblock {\em J. Parallel Distrib. Comput.}, 77(C):84–94, 2015.

\bibitem{dist_rw_2009}
Atish Das~Sarma, Danupon Nanongkai, and Gopal Pandurangan.
\newblock Fast distributed random walks.
\newblock In {\em PODC}, page 161–170, 2009.

\bibitem{metapath2vec}
Yuxiao Dong, Nitesh~V. Chawla, and Ananthram Swami.
\newblock Metapath2vec: Scalable representation learning for heterogeneous
  networks.
\newblock In {\em KDD}, pages 135--144, 2017.

\bibitem{FORTUNATO201075}
Santo Fortunato.
\newblock Community detection in graphs.
\newblock {\em Physics Reports}, 486(3):75 -- 174, 2010.

\bibitem{edge2vec}
Zheng Gao, Gang Fu, Chunping Ouyang, Satoshi Tsutsui, Xiaozhong Liu, Jeremy~J.
  Yang, Christopher~R Gessner, Brian Foote, David~J. Wild, Ying Ding, and
  Qi~Yu.
\newblock edge2vec: Representation learning using edge semantics for biomedical
  knowledge discovery.
\newblock {\em BMC Bioinformatics}, 20, 2019.

\bibitem{gelmanbda04}
Andrew Gelman, John~B. Carlin, Hal~S. Stern, and Donald~B. Rubin.
\newblock {\em Bayesian Data Analysis}.
\newblock Chapman and Hall/CRC, 2nd ed. edition, 2004.

\bibitem{grimmett2001probability}
Geoffrey Grimmett and David Stirzaker.
\newblock {\em Probability and random processes}, volume~80.
\newblock Oxford university press, 2001.

\bibitem{grover2016node2vec}
Aditya Grover and Jure Leskovec.
\newblock node2vec: Scalable feature learning for networks.
\newblock In {\em KDD}, pages 855--864, 2016.

\bibitem{HOLDEN1998371}
Lars Holden.
\newblock Geometric convergence of the metropolis-hastings simulation
  algorithm.
\newblock {\em Statistics \& Probability Letters}, 39(4):371 -- 377, 1998.

\bibitem{8663393}
S.~{Ji}, N.~{Satish}, S.~{Li}, and P.~K. {Dubey}.
\newblock Parallelizing word2vec in shared and distributed memory.
\newblock {\em TPDS}, 30(9):2090--2100, 2019.

\bibitem{kipf2016semi}
Thomas~N. Kipf and Max Welling.
\newblock Semi-supervised classification with graph convolutional networks.
\newblock In {\em ICLR}, 2017.

\bibitem{Kullback51klDivergence}
S.~Kullback and R.~A. Leibler.
\newblock On information and sufficiency.
\newblock {\em Ann. Math. Statist.}, 22(1):79--86, 1951.

\bibitem{kyrola2013drunkardmob}
Aapo Kyrola.
\newblock Drunkardmob: billions of random walks on just a pc.
\newblock {\em RecSys}, pages 257--264, 2013.

\bibitem{mlsys2019_71}
Adam Lerer, Ledell Wu, Jiajun Shen, Timothee Lacroix, Luca Wehrstedt, Abhijit
  Bose, and Alex Peysakhovich.
\newblock Pytorch-biggraph: A large scale graph embedding system.
\newblock In {\em SysML}, pages 120--131. 2019.

\bibitem{7113345}
R.~{Li}, J.~X. {Yu}, L.~{Qin}, R.~{Mao}, and T.~{Jin}.
\newblock On random walk based graph sampling.
\newblock In {\em ICDE}, pages 927--938, 2015.

\bibitem{lovasz1993random}
L.~Lov{\'a}sz.
\newblock Random walks on graphs: A survey.
\newblock {\em Combinatorics, Paul Erdos is Eighty}, 2(1):1--46, 1993.

\bibitem{10.1145/1807167.1807184}
Grzegorz Malewicz, Matthew~H. Austern, Aart~J.C Bik, James~C. Dehnert, Ilan
  Horn, Naty Leiser, and Grzegorz Czajkowski.
\newblock Pregel: A system for large-scale graph processing.
\newblock In {\em SIGMOD}, pages 135--146, 2010.

\bibitem{direct_sampling}
George Marsaglia.
\newblock Generating discrete random variables in a computer.
\newblock {\em Commun. ACM}, 6(1):37--38, 1963.

\bibitem{mengersen1996rates}
Kerrie~L Mengersen, Richard~L Tweedie, et~al.
\newblock Rates of convergence of the hastings and metropolis algorithms.
\newblock {\em The annals of Statistics}, 24(1):101--121, 1996.

\bibitem{Metropolis1953}
Nicholas Metropolis, Arianna~W. Rosenbluth, Marshall~N. Rosenbluth, Augusta~H.
  Teller, and Edward Teller.
\newblock Equation of state calculations by fast computing machines.
\newblock {\em The Journal of Chemical Physics}, 21(6):1087--1092, 1953.

\bibitem{mikolov2013efficient}
Tomas Mikolov, Kai Chen, Greg Corrado, and Jeffrey Dean.
\newblock Efficient estimation of word representations in vector space.
\newblock {\em arXiv preprint}, 2013.

\bibitem{mikolov2013distributed}
Tomas Mikolov, Ilya Sutskever, Kai Chen, Greg~S Corrado, and Jeff Dean.
\newblock Distributed representations of words and phrases and their
  compositionality.
\newblock In {\em NIPS}, pages 3111--3119, 2013.

\bibitem{OPSAHL2009155}
Tore Opsahl and Pietro Panzarasa.
\newblock Clustering in weighted networks.
\newblock {\em Social Networks}, 31(2):155 -- 163, 2009.

\bibitem{10.1145/2983323.2983361}
Erik Ordentlich, Lee Yang, Andy Feng, Peter Cnudde, Mihajlo Grbovic, Nemanja
  Djuric, Vladan Radosavljevic, and Gavin Owens.
\newblock Network-efficient distributed word2vec training system for large
  vocabularies.
\newblock In {\em CIKM}, pages 1139--1148, 2016.

\bibitem{perozzi2014deepwalk}
Bryan Perozzi, Rami Al-Rfou, and Steven Skiena.
\newblock Deepwalk: Online learning of social representations.
\newblock In {\em KDD}, pages 701--710, 2014.

\bibitem{fairwalk}
Tahleen Rahman, Bartlomiej Surma, Michael Backes, and Yang Zhang.
\newblock Fairwalk: Towards fair graph embedding.
\newblock In {\em IJCAI}, pages 3289--3295, 2019.

\bibitem{roberts1994simple}
Gareth~O Roberts and Adrian~FM Smith.
\newblock Simple conditions for the convergence of the gibbs sampler and
  metropolis-hastings algorithms.
\newblock {\em Stochastic processes and their applications}, 49(2):207--216,
  1994.

\bibitem{6823116}
Y.~{Shao}, B.~{Cui}, and L.~{Ma}.
\newblock Page: A partition aware engine for parallel graph computation.
\newblock {\em TKDE}, 27(2):518--530, 2015.

\bibitem{ma_sigmod2020}
Yingxia Shao, Shiyue Huang, Xupeng Miao, Bin Cui, and Lei Chen.
\newblock Memory-aware framework for efficient second-order random walk on
  large graphs.
\newblock In {\em SIGMOD}, page 1797–1812, 2020.

\bibitem{tsitsulin2018verse}
Anton Tsitsulin, Davide Mottin, Panagiotis Karras, and Emmanuel M{\"u}ller.
\newblock Verse: Versatile graph embeddings from similarity measures.
\newblock In {\em WWW}, pages 539--548, 2018.

\bibitem{walker1977efficient}
Alastair~J Walker.
\newblock An efficient method for generating discrete random variables with
  general distributions.
\newblock {\em TOMS}, 3(3):253--256, 1977.

\bibitem{yang2019knightking}
Ke~Yang, MingXing Zhang, Kang Chen, Xiaosong Ma, Yang Bai, and Yong Jiang.
\newblock Knightking: a fast distributed graph random walk engine.
\newblock In {\em SOSP}, pages 524--537, 2019.

\bibitem{zhang2018link}
Muhan Zhang and Yixin Chen.
\newblock Link prediction based on graph neural networks.
\newblock In {\em NIPS}, pages 5165--5175, 2018.

\bibitem{zhou2018efficient}
Dongyan Zhou, Songjie Niu, and Shimin Chen.
\newblock Efficient graph computation for node2vec.
\newblock {\em arXiv preprint}, 2018.

\end{thebibliography}
